\documentclass[12pt]{article}
\usepackage{amsmath}
\usepackage{graphicx,psfrag,epsf}
\usepackage{enumerate}
\usepackage{natbib}
\usepackage{amssymb}
\usepackage{bigints}
\usepackage{cite}
\usepackage{algorithm}
\usepackage[noend]{algpseudocode}
\usepackage{amsthm}
\usepackage{bm}
\usepackage{booktabs}
\usepackage[labelfont=bf,textfont=md]{caption}
\usepackage{threeparttable}
\usepackage[colorlinks, citecolor = blue, linkcolor = black]{hyperref}
\usepackage{authblk}

\newcommand{\blind}{0}

\numberwithin{equation}{section}
\newtheorem{prop}{Proposition}

\begin{document}

\def\spacingset#1{\renewcommand{\baselinestretch}%
{#1}\small\normalsize} \spacingset{1}
\def\lf{\left\lfloor}   
\def\rf{\right\rfloor}
\def\intR{\int_{-\infty}^\infty}
\def\Cov{\text{Cov}}
\def\iid{\stackrel{iid}{\sim}}
\def\V{\text{Var}}
\def\E{\mathbb{E}}
\def\vp{\varphi}
\def\df{\mathcal{D}}
\def\bX{\bm{X}}
\def\bx{\bm{x}}


\if0\blind
{
  \title{\bf Locally Optimized Random Forests}
  \author[1,2]{Tim Coleman}
  \author[1]{Kimberly Kaufeld}
  \author[1]{Mary Frances Dorn}
  \author[2]{Lucas Mentch}
    \affil[1]{\footnotesize Los Alamos National Laboratory, Statistical Sciences Group, Los Alamos, NM 87544, USA}
    \affil[2]{\footnotesize University of Pittsburgh, Department of Statistics, Pittsburgh, PA 15215, USA}
  \maketitle
} \fi

\if1\blind
{
  \bigskip
  \bigskip
  \bigskip
  \begin{center}
  {\LARGE\bf A Distributional Approach to Random Forest Weighting, with Applications to Hurricane Outage Forecasting}
\end{center}
  \medskip
} \fi

\bigskip
\begin{abstract}
 \noindent Standard supervised learning procedures are validated against a test set that is assumed to have come from the same distribution as the training data. However, in many problems, the test data may have come from a different distribution. We consider the case of having many labeled observations from one distribution, $P_1$, and making predictions at unlabeled points that come from $P_2$. We combine the high predictive accuracy of random forests \citep{Breiman2001} with an importance sampling scheme, where the splits and predictions of the base-trees are done in a weighted manner, which we call \textit{Locally Optimized Random Forests}. These weights correspond to a non-parametric estimate of the likelihood ratio between the training and test distributions. To estimate these ratios with an unlabeled test set, we make the \textit{covariate shift} assumption, where the differences in distribution are only a function of the training distributions \citep{Shimodaira2000}. This methodology is motivated by the problem of forecasting power outages during hurricanes. The extreme nature of the most devastating hurricanes means that typical validation set ups will overly favor less extreme storms. Our method provides a data-driven means of adapting a machine learning method to deal with extreme events.
\end{abstract}

\noindent%
{\it Keywords:} random forest, importance sampling, semi-supervised learning, covariate shift

\spacingset{1}
\section{Introduction}
\label{sec:intro}

In machine learning, it is often assumed, implicitly or explicitly, that data used in training and data held out for prediction follow the same distribution. As such, models find an approximating function $\hat{f}$ that minimizes the \textit{global generalization error}, which for a loss function $L(\hat{f}(\bX), Y)$ is defined as $\E_{(\bX,Y)} L(\hat{f}(\bX), Y)$, where the expectation is taken with respect to the distribution of both $\bX$ and $Y$. However, it may be that
\[
\E_{(\bX,Y) \sim P_{\text{train}}} L(\hat{f}(\bX), Y) \neq \E_{(\bX,Y) \sim P_{\text{test}}} L(\hat{f}(\bX), Y)
\]
because $P_\text{train} \neq P_\text{test}$. As such, minimizing the left hand side may not yield an estimator that minimizes the second quantity. This idea of utilizing knowledge of where predictions will be sought as part of the training process is a natural fit in areas such as personalized medicine \citep{Liu2016}, for example, where physicians may often seek the most accurate predicted outcomes for particular patients, rather than a global minimizer.  \citet{Powers2015} make use of this notion of \emph{customized training} to cluster pixels from mass spectrometric images taken from lung cancer patients in order to fit more precise individual models to each cluster.

To formalize the above framework, consider covariates $\bX$ which take values in some $p$ dimensional space $\mathcal{X} \subset \mathbb{R}^p$ and a response Y which takes values in $\mathcal{Y} \subset \mathbb{R}$. Suppose we have two sets of data $\df$ and $\df'$ where $\df = (\bX_i, Y_i)_{i=1}^n \iid P_1$ and $\df' = (\bX_i', Y_i')_{i=1}^m \iid P_2$. where $P_i$ is a probability measure on $\mathcal{X} \times \mathcal{Y}$ for $i = 1, 2$. 
Furthermore, assume that the $Y_i'$ have been censored - and the goal is to attain accurate point estimates and prediction intervals for $Y_i'$. Now, suppose $P_1, P_2$ satisfy
\begin{equation} \label{eqn:dist_fact}
\begin{split}
    P_1(\bX,Y) &= P(Y|\bX)P^*_1(\bX) \\
    P_2(\bX,Y) &= P(Y|\bX)P^*_2(\bX)
\end{split}
\end{equation}
 so that the conditional distribution of the target is the same for both datasets - the change is in the covariate distribution. This model is commonly referred to as the \textit{covariate shift} model, and has been the source of intense research in recent decades \citep{Shimodaira2000, Sugiyama2005, Sugiyama2007, Reddi2015}. The issue arises when $P_2^*$ and $P_1^*$ concentrate mass in different areas of $\mathcal{X}$. In this case, standard guarantees about the effectiveness of many regression estimates of $P(Y|\bX = \bx)$ are invalid for $\bx$ in areas of low mass of $P_1^*$, even as $n\to\infty$. This is especially problematic if the low mass areas of $P_1^*$ have high mass in $P_2^*$. To resolve this, we propose learning a mapping between $P_1^*$ and $P_2^*$ by estimating the likelihood ratio function $\ell(\bX) = \frac{dP_2^*(\bX)}{dP_1^*(\bX)}$. Note we have assumed that $P_1^*$ and $P_2^*$ are absolutely continuous with respect to each other, i.e. for all measurable $A$, $P_1^*(A) > 0 \iff P_2^*(A) > 0$. 
In essence, we want to calculate the likelihood ratio, $\Lambda = \frac{dP_2^*}{dP_1^*}$, without necessarily specifying the form of $P_1^*$ and $P_2^*$. This precludes the use of typical parametric likelihood functions. Moreover, the high dimension of many problems means that the naive approach of estimating two densities will be unstable.

\subsection{A Motivating Example: Hurricane Power Outages}
One of the most damaging effects of hurricanes is the loss of power for many people in the storm track. Forecasting these outage counts is a direct way of quantifying the damage done by a hurricane, whereas meteorological forecasts, such as of windspeed and storm surge, tend to focus less on the human impact of the storm. Advances in machine learning have led to large improvements in predictive modeling of power outages that result from tropical storms and hurricanes. These models typically take in two sets of covariate information: (1) Information about the storm, such as windspeed expected in each study unit (2) information about each study unit, such as the soil types and demographics of the unit.  

The focus of this paper is to develop a method for accurately forecasting outages during storms across a wide variety of geographic extents, using only inputs available on such a geographic scale. Effectively, this means we cannot use information about the power-grid itself due to limited coverage, resolution, and types of information reported about each local grid. Several challenges are inherent to this problem:
\begin{description}
\item[Data Availability] The National Hurricane Center \citep{NHC2013} only provides full data for storms from 1995 onwards. 

\item[Rarity of Severe Events] Severe storms are, by definition, anomalous, and therefore are potentially underrepresented in the available data. Moreover, they may be overrepresented for particular areas of interest due to chance. 

\item[Interest in Severe Events] Forecasting less severe outages is inherently less useful to practitioners - often, the interest is in whether or not the forecasts for the big storms are accurate. 
\end{description}

Outage data is provided by the EAGLE-I system, which aggregates national information about the power-grid.  Power outages are clearly dynamic throughout the storm - in our dataset, outages are reported every 15 minutes for each county affected for the duration of the storm. For simplicity, we summarize the outage extent in the following way: (1) We record a running minimum outage $M_{i,t} = \min\{O_{i,k}: \ k \in [t, t+8)\}$, where $O_{i,k}$ is the time series of power outages in county $i$; (2) We let $Y_i = \log_{10}(\max_t M_{i,t})$. This quantity serves as our response variable, and is referenced with the predictors listed in \autoref{appdx:DetailedTab} and in \citet{Pasqualini2017}. Taking the logarithm of the outages helps to alleviate the heavytailed nature of the response, and further its interpretation can help quantify the magnitude of the expected effect \citep{Tokdar2010, Willoughby2007}.

In all, the data contains outage counts from 17 hurricanes and tropical storms between 2011 and 2017, for a total of 5015 observations, on 75 predictors. Given a county in a storm with covariates $\bx$, we want to estimate the conditional distribution $Y| \bX = \bx$ of county level outages, with emphasis on point estimates and prediction intervals. Moreover, we are typically interested in making forecasts for the entire affected region of a hurricane at once.

\section{Related Work}
Now, we describe related research regarding both the covariate shift model and the hurricane forecasting problem.
\subsection{Related Weighted Random Forest Work}
To fit a random forest into this framework, one solution would be to implement a weighted bootstrap in the resampling phase of the forest. Canonically, each observation has probability of being selected $p_i \equiv 1/n$, under the weighted scheme, $p_i \propto w_i$, where $w_i$ are some weights obtained \textit{a priori}. This approach was considered by \citet{Xu2016}, who proposed the following weighting scheme:
\begin{enumerate}
    \item Train a random forest on the original data set, in the process constructing $T_1, ..., T_B$. 
    \item Let $\bx_0$ be the test point of interest. Pass $\bx_0$ down the each tree. Let 
    \[
    S_{ij}(\bx_0) = I(\text{$\bx_0$ shares a terminal node with $\bX_i$ in $T_j$}) 
    \]
    and then define $D_i(\bx_0) = \frac{\sum_{j=1}^B S_{ij}(\bx_0)}{\sum_{i=1}^n \sum_{j=1}^B S_{ij}(\bx_0)}$. 
    \item In resampling, draw each observation with probability $p_i(\bx_0) = D_i(\bx_0)$ with replacement.
\end{enumerate}
This method bears structural similarity to our proposed method in that it weights the training observations by their similarity to the test point. Note that $D_i(\bx_0)$ is similar to the proximity metric \citep{Breiman2001, Friedman2001b} associated with random forests, which can be used as an adaptive distance metric. We note that this approach is well-suited for making predictions at a single point, i.e. where $P_2^*$ is a degenerate distribution with all of its mass concentrated at $\bx_0$. However, the weights used change from test point to test point, meaning that a new weighting scheme must be used for each point, and thus a new random forest must be trained for each test point, leading to a total of $|\df'| + 1$ forests needed. This may incur needless computational cost. A speed-up could be to cluster the test points and then apply the above scheme to the centroids of the clusters to get a weighting scheme for all points within the cluster. This is quite similar to the approach suggested by \citet{Powers2015}. In contrast to these procedures, we want to use distributional information about the covariates in our weighting. Moreover, for practical purposes, we seek a method with minimal additional computational overhead.

\subsection{Related Hurricane Outage Work}
\citet{Liu2005} used negative binomial regression to forecast outages during three storms during the 1990's. \citet{Guikema2012} found that generalized linear models lacked sufficient flexibility to accurately forecast power outages, and instead turned to non-parametric models, such as random forests and gradient boosting. More recently, \citet{Wanik2015} used a combined random forest, gradient boosting, and a single decision tree to forecast outages. \citet{He2017} used quantile regression forests \citep{Meinshausen2006} to provide prediction intervals, in addition to point estimates, for power outage forecasts. Quantile based methods may be preferable due to the heavy-tailed nature of power outage distributions - the averaging used in conditional mean estimation can lead to severe over/under estimates of power outages. Moreover, practitioners are likely more interested in a prediction interval than a confidence interval, as a prediction interval can inform evacuations/preparations. As such, much of the recent work in random forest inference, such as \citet{Wager2014, Mentch2016, Wager2017, Mentch2017, Coleman2019, Peng2019} is of less interest because of their focus on conditional mean estimation.

In our case, assume we train a model only on data from $P_1$, which may be data from several hurricanes in years prior, and then use it to make predictions about data that come from $P_2$, such as the outages for a yet-observed hurricane, whose characteristics may be quite different than storms previously recorded. \autoref{tab:unw_metrics} shows the result of this procedure for 6 hurricanes between 2012 and 2017. In particular, each model is tuned by minimizing the out-of-bag error for each parameter configuration, and the optimal model is then used to make predictions for the held-out storm. For example, to forecast Hurricane Arthur, we use data from the 16 other storms to train a random forest, which is then used to learn $f(\bx) = \E(Y |\bX = \bx)$, and $Q_\alpha(\bx) = F^{-1}_Y(\alpha|\bX = \bx)$ for $\alpha = 0.1, 0.5, 0.9$. Thus, the forest predicts the conditional mean, the conditional median, and a conditional 80 \% prediction interval. If the covariate structure was similar for each storm, we would anticipate seeing roughly similar error metrics across storms, especially seeing as the sample size is similar across each iteration. Rather, we see that three storms (Harvey, Nate, Matthew) have similar error metrics, while Arthur, Irma, and Sandy are much higher. It is not surprising that these are the storms that are most difficult to forecast - Irma and Sandy in particular were historically damaging storms \citep{IrmaSummary}. Perhaps more telling is that the prediction intervals for the higher error storms provide much poorer coverage. \citet{Meinshausen2006} showed that, under regularity assumptions, the conditional quantiles estimated by a quantile regression forest are consistent - as such, we would expect prediction intervals to maintain near the nominal coverage level. However, Harvey shows minor departures from this coverage level and Irma, Sandy, and Arthur shows a extreme departure from this level. %
To summarize performance, we also report a ``score" metric, which is defined as
\begin{equation} \label{eqn:score}
\text{Score} = \left(\frac{1}{\text{MAE}} + \frac{1}{\text{RMSE}} + \frac{4}{\text{IntWidth}}\right)\frac{\text{Covg}}{1-\alpha}
\end{equation}
so that the score is penalized for higher loss (MAE, RMSE), for wider intervals, and for lower coverage \%. This is not a formal loss function, but an attempt to quantify overall predictive performance. We note that Irma and Sandy have the lowest scores by far - again suggesting the difficulty in forecasting the damage from these storms.
\begin{table}[H]
\centering
\begingroup\footnotesize
\begin{tabular}{lllrrrrr}
  \hline
Storm & \texttt{mtry} & \texttt{nodesize} & MAE & RMSE & Covg & Interval Width & Score \\ 
  \hline
Matthew-2016 & 50 & 5 & 0.6269 & 0.7861 & 0.8898 & 2.6946 & 4.3021 \\ 
  Nate-2017 & 40 & 5 & 0.6727 & 0.8124 & 0.8759 & 2.5094 & 4.1960 \\ 
  Harvey-2017 & 50 & 5 & 0.7509 & 0.9026 & 0.7632 & 2.4214 & 3.4695 \\ 
  Arthur-2014 & 45 & 5 & 0.8498 & 1.0322 & 0.6862 & 2.2623 & 2.9839 \\ 
  Sandy-2012 & 40 & 10 & 0.9817 & 1.2197 & 0.5781 & 2.2376 & 2.3293 \\ 
  Irma-2017 & 45 & 5 & 1.1846 & 1.4044 & 0.3706 & 2.4051 & 1.3258 \\ 
   \hline
\end{tabular}
\caption{Tuned random forest results for 6 storms in the hurricane dataset. ``Covg" and ``Interval Width" refer to 80\% prediction intervals.}
\label{tab:unw_metrics}
\endgroup
\end{table}

\begin{figure}[H]
    \centering
    \includegraphics[width = .75\textwidth]{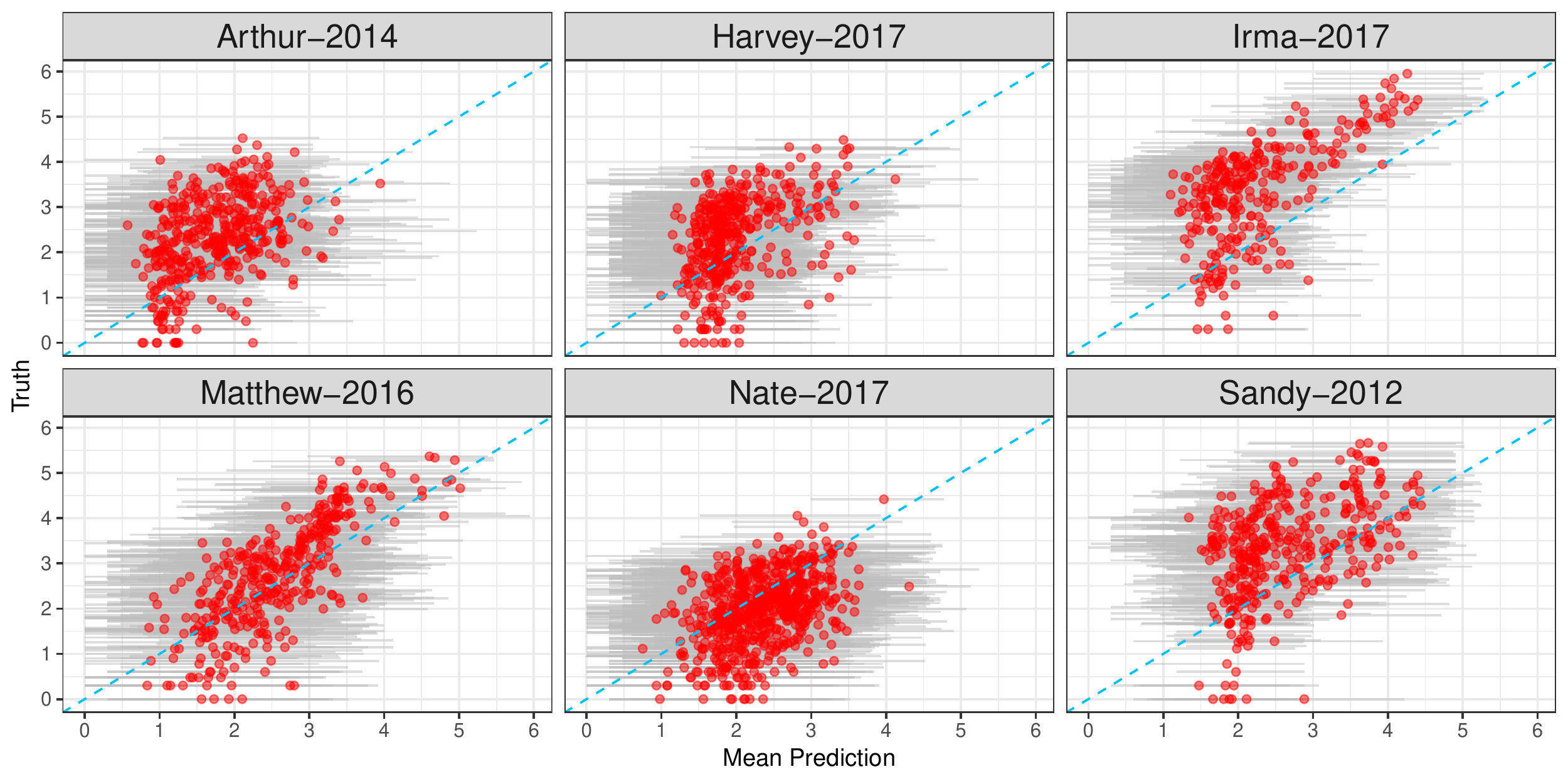}
    \caption{Fitted vs Predicted for each storm-holdout model. Blue line represents perfect prediction, and grey bars represent 80\% prediction intervals}
    \label{fig:plottedstorms}
\end{figure}

\section{Methods}
\label{sec:meth}

We begin with a brief summary of importance sampling. Importance sampling refers to weighting observations to either reduce the variance of some point estimate or to ``tilt" a sample observed from $P_1$ to be similar to $P_2$. 
As such, importance sampling seeks to weight each $X$ by how much it resembles a sample from $P_2$. The idea is to replace the observations $X$ with $X^* = Xw(X)$, for $w(x) = \frac{P_2(x)}{P_1(x)}$. We then let $\tilde{\mu} = \frac{1}{n}\sum_{i=1}^n X^*_i$. If, $P_1$ and $P_2$ are known, the $w(X_i)$ are already normalized (in the sense that they sum to 1). In our case, we know neither distribution, and can only calculate an un-normalized likelihood ratio between the two. As such, the self-normalized importance sampling estimate $ \tilde{\mu} = \frac{\sum_{i=1}^n w_i f(X_i)}{\sum_{j=1}^n w_j}$ is of more use.  The problem is to construct a random forest using data from $P_1$ as if the data had come from $P_2$. We propose a two stage procedure to solve this problem:
\begin{enumerate}
    \item First, we train a model to learn $\ell(\bx) = \frac{dP^*_2(\bx)}{dP^*_1(\bx)}$, the ratio of the data densities at $\bx$. We then estimate $\ell(\bX_1), ..., \ell(\bX_n)$ for each point in $\df$.
    \item We construct a randomized tree using an importance weighted criterion for both the splits and the predictions.
\end{enumerate}
Tree-based models are constructed by recursively partitioning the feature space. Partitioning takes a rectangular subspace $A$ and partitions it into two further rectangular subspaces $A_L, A_R$, where $A_L = \{ X \in A : X_i^{(j)} < z\}$, $A_R = A \setminus A_L$, and $x^{(j)}$ represents the $j^{\text{th}}$ coordinate of an observation. In the context of a continuous feature space (i.e. no categorical predictors), the quality of a split is assessed by:
\begin{multline} \label{eqn:CART}
        L(j, z) = \frac{1}{N_n(A)}\sum_{i=1}^n (Y_i - \bar{Y}_A)^2I(\bX_i \in A)\ - \\
         \frac{1}{N_n(A)}\sum_{i=1}^n \big(Y_i - \bar{Y}_{A_L}I(X_i^{(j)} < z) - \bar{Y}_{A_R}I(X_i^{(j)} \geq z) \big)^2I(\bX_i \in A)
\end{multline}
where $N_n(A)$ indicates the number of observations in the original sample that lie in region $A$ and $\bar{Y}_A$ is the sample mean of the response over all observations who lie in region $A$. This criterion is typically evaluated at all possible split points, and the split selected satisfies $(A_L, A_R) = \text{argmax}_{(j,z)} L(j,z)$. This process is initialized with $A = \mathcal{X}$, and then repeated recursively until the trees reach a specified depth or terminal node size. The trees output a rectangular partition, $A_1, ..., A_m$ where $m$ is the number of terminal nodes in the tree, and where $\mathcal{X} = \cup_{i=1}^m A_i$. Let $A^*(\bx)$ be the partition segment containing $\bx$, so that the prediction at $\bx$ is given by
\[
T(\bx; \df) = \sum_{i=1}^n \frac{I(\bX_i \in A^*(\bx))}{N_n(A^*(\bx))} Y_i.
\]
The construction of the trees above can be seen as repeated calculation of different statistical functionals. For a given probability measure $P$ supported on a set $A$, consider a rectangular partition of $A$ into $A_L$ and $A_R$, such that $A_L = \{ \bx \in A : x^{(j)} < z\}$ and $A_R = A \setminus A_L$. Define $P_L = \frac{1}{P(A_L)}P I(\bx \in A_L)$, normalizing so that $P_L$ is a valid probability measure. We can then define the following functionals
\begin{align*}
    T_1(P) &= \int y dP(y) \\
    \begin{split}
        T_{j,z}(P) &= \int (y - T_1(P))^2 dP(y) \ - \\
        &\hspace{8mm} \int\big[(y - T_1(P_L))^2I(\bx \in A_L) +(y - T_1(P_R))^2I(\bx \in A_R)\big] dP(y).
    \end{split}
\end{align*}
In the above, the functionals are calculated only with respect to the response coordinate - i.e. they are scalars, not vectors. For a given node $A$, define $\hat{P}_A = \frac{1}{N_n(A)}\sum_{i=1}^{n} \delta_{(\bX_i, Y_i)}I(\bX_i \in A)$, where $\delta_{(\bX_i, Y_i)}$ places mass 1 at the pair $(\bX_i, Y_i)$. 
We can redefine \autoref{eqn:CART} in terms of functionals of empirical distributions as
 \begin{equation*}
     L(j, z) = T_{j,z}(\hat{P}_A).
 \end{equation*}
The prediction stage can similarly be seen as $T(\bx; \df) = T_1(\hat{P}_{A^*(\bx)})$. The main innovation we propose here is to replace $\hat{P}_A$, which may estimate the training data distribution well, with another estimate $\tilde{P}_A$ that well approximates the distribution of the test data. Then, the functionals described above are calculated over $\tilde{P}_A$ for both the structure and prediction stages of the tree construction. In practice, we use the following formulation of $\tilde{P}_A$, which depends on a weight vector $\bm{w}$
\begin{equation}\label{eqn:imp_samp}
    \tilde{P}_{A, \bm{w}} = \sum_{i=1}^n \frac{w_iI(\bX_i \in A)}{\sum_{j=1}^n w_j I(\bX_j \in A)} \delta_{(\bX_i, Y_i)}.
\end{equation}
We thus replace the factor $1/N_n(A)$ with a value proportional to $w_i$. We use $\bm{w} = \{\ell(\bX_1),..., \ell(\bX_n)\}$, so that $T(\tilde{P}_{\bm{w}})$ is an approximation to $T(P_2)$ rather than $T(P_1)$. Tree construction proceeds by recursively maximizing $T_{j,z}(\tilde{P}_{A, \bm{w}})$ over each node, until the control parameters of the tree are met. 
As in the unweighted case, we can restrict the set of possible splits randomly at each node, such as only allowing $\texttt{mtry} < p$ features available for splitting, which can provide a forest variance reduction by decorrelating the trees. Then, the weighted tree predictions are given as
\[
T_{\bm{w}}(\bx ; \df) = T_1(\tilde{P}_{A^*(\bx), \bm{w}}).
\]

Finally, a forest is created by resampling the data many times and training a randomized tree on each data. The forest prediction, like in standard random forests (which estimate the conditional mean fuction) is given by
\[
m_{B,\bm{w}}(\bx ; \df)  = \frac{1}{B}\sum_{k=1}^B T_{\bm{w}}(\bx ; \df, \xi_k)
\]
where $\xi_k$ are iid randomization parameters determining the resamples and available features for splitting at each node. These procedures are summarised in \autoref{alg:weightedTree} for the weighted tree and \autoref{alg:weightedRF} for the entire forest.

\begin{algorithm}[ht]
\caption{Weighted Regression Tree}\label{alg:weightedTree}
\begin{algorithmic}[1]
\Procedure{WeightedTree}{$\df, \bm{w}, \xi, m_n$} \Comment{$\bm{w}$ are weights, $\xi$ is randomization, $m_n$ is maximum number of terminal nodes}
\State Set $\mathcal{P}_0 = \{\mathcal{X}\}$, $t=1$, and $d = 0$ \Comment{The root node $\mathcal{P}_0$ is the entire feature space, we start with $t=1$ terminal nodes, $d = 0$ is the depth}
\State For all $1 \leq k \leq \texttt{nrow}(\mathcal{D})$ set $\mathcal{P}_k = \varnothing$
\While{$t < m_n$}
\If{$\mathcal{P}_d = \varnothing$}
\State $d \leftarrow d + 1$
\Else
\State Set $A$ as the first element in $\mathcal{P}_d$ \Comment{$P_A$ is the within-node distribution}
\State Let $\mathcal{M}_{\xi,d} \subset \{1,.., p\}$ be features available for splitting 
\State Evaluate $T_{j,z}(\tilde{P}_{A,\bm{w}}) \ \forall z $ and for all $j \in \mathcal{M}_{\xi,d}$
\State Set $A_L^* = \{ \bX \in A : X^{(j^*)} < z^*\}$ where $z^*, j^* = \text{argmax}_{j,z}(P_A)$ and set $A_R^* = A \setminus A_L^*$
\State Set $\mathcal{P}_d \leftarrow \mathcal{P}_d \setminus \{A\}$ and $\mathcal{P}_{d+1} \leftarrow \mathcal{P}_{d+1}\cup \{A^*_L\} \cup \{A^*_R\}$
\State Set $t \leftarrow t + 1$
\EndIf
\EndWhile
\State Prediction at point $\bx$ is made by $T_1(\tilde{P}_{A^*(\bx), \bm{w}})$ where $A(\bx) \in \mathcal{P}_d$ is the node containing $\bx$
\EndProcedure
\end{algorithmic}
\end{algorithm}

\begin{algorithm}[ht]
\caption{Locally Optimized Random Forest}\label{alg:weightedRF}
\begin{algorithmic}[1]
\Procedure{LocalRF}{$\df_{\text{TRAIN}}, \df_{\text{TEST}}, \texttt{REPLACE}, k_n, B$}
\State Apply method of \citet{Kanamori2009} to generate $\hat{\ell}_i$
\For{$k \in \{1, ..., B\}$}\Comment{$B$ is total number of trees to be trained}
\If{ \texttt{REPLACE}}
\State Draw $k_n$ observations w/ replacement
\Else
\State Draw $k_n$ observations w/o replacement
\EndIf
\State Let $\df_{k, k_n}$ be the resampled data, and $\bm{\ell}_{k, k_n}$ be the resampled weights 
\State Set $T_k \leftarrow \textsc{WeightedTree}(\df_{k, k_n}, \bm{\ell}_{k, k_n}, \xi_k)$ \Comment{$\xi_k$ controls other randomization}
\EndFor
\State \textbf{return} $\{T_1, ..., T_B\}$ \Comment{Collection of trees}
\EndProcedure
\end{algorithmic}
\end{algorithm}
\subsection{Weighted Quantile Regression}
Recall that a major interest in the forecasting problem is the inclusion of prediction intervals, and quantile regression forests \citep{Meinshausen2006} provide a natural means of non-parametric quantile regression. As such, we propose a means of using the importance forest procedure for quantile regression. As \citet{Meinshausen2006} notes, a random forest estimate can be reformulated as a weighted mean of the observations, as opposed to the sample mean of the trees. For a prediction point $\bx$ and a point in the training set $\bX_i$, a decision tree (constructed using prior weights $\bm{w}$) drawn with parameter $\xi$ induces the following weights
\[
t_i(\bx; \xi, \bm{w}) = I(\bX_i \in A_\xi^*(\bx))\frac{w_i}{\sum_{j=1}^n w_j I(\bX_j \in A_\xi^*(\bx))}.
\]
Then, given $B$ trees trained using randomization parameters $\xi_1,..., \xi_B$, we can define the random forest weights by
\[
r_{i,B}(\bx; \bm{w}) = \frac{1}{B}\sum_{k=1}^B t_i(\bx; \xi_k, \bm{w}).
\]
Following \citet{Meinshausen2006}, we can then use these weights to get an estimate of $F(y | \bX = \bx) = P(Y \leq y | \bX = \bx)$ as 
\[
\tilde{F}_{\bm{w}}(y | \bX = \bx) = \sum_{i=1}^n r_{i,B}(\bx; \bm{w}) I(Y_i \leq y).
\]
We can similarly define a conditional quantile function $\tilde{Q}_{p,\bm{w}}(\bx) = \inf\{ y : \ \tilde{F}_{\bm{w}}(y | \bX = \bx) \geq p \}$. Note that $\tilde{F}_{\bm{w}}(y | \bX = \bx)$ only takes on $n + 1$ values, so evaluation of $\tilde{Q}_{p, \bm{w}}(\bx)$ amounts to a grid search over these $n+1$ values. For a provided quantile, $p$, we see that $\tilde{Q}_{p, \bm{w}}(\bx) = Y_{k^*}$, where $k^* = \min_k \sum_{i=1}^k r_{(i), B}(\bx; \bm{w}) \geq p$, where the notation $r_{(i),B}(\bx ; \bm{w})$ indicates that the RF weights are now ordered by magnitude of the response value, i.e. $i > k \iff Y_i \geq Y_k$.
 
\subsection{Learning \texorpdfstring{$\ell$}{l}}
\label{subsec:lrt_learn}
Each element of the weight vector $\ell(\bX_i)$ is a ratio of densities of two different covariate distributions. These densities are unknown and are over high dimensional feature space. As such, many traditional density estimation tools are unlikely to be effective. We describe two candidate procedures for density estimation, probabilistic classification and kernel moment matching. We argue that the probabilistic classification approach, while simple to implement, may be unstable in high dimensions. 

\subsubsection{Probabilistic Classification}
We can use the favorable properties of tree based density estimates in high dimensions to learn $\ell$. The algorithm of \citet{Breiman2001} can be used for unsupervised learning, by returning measures of adaptive distance between observations. Crucially, this procedure relies on the creation of a synthetic covariate dataset, and then learning the probability that a particular observation came from the true or synthetic dataset. The synthetic dataset is created by drawing $n$ observations (with replacement) uniformly and independently from each covariate, destroying any dependencies between the observations. The idea is that if there is high-dimensional structure, the model should easily discriminate between the two datasets. In the covariate shift literature, this procedure is referred to as a \textit{probabilistic classification} method, as it transform the density ratio estimation problem into a classification problem \citep{Barber2019}.

To formalize the above, we impose another assumption about the distribution of test and training. For all $\bX_i \in \{\df, \df'\}$, we assume that $\bX_i \iid P(\bX_i) = \alpha P_1^*(\bX_i) + (1-\alpha)P_2^*(\bX_i)$, where $\alpha \in (0,1)$. In the canonical machine learning context, $\alpha \equiv 1$ (without loss of generality), which covers the situation where the test and training covariates have the same distribution. We introduce the synthetic response $Z = I(X \sim P_2^*)$. For every observation in $\bX_i \in \{\df, \df'\}$, this amounts to $Z_i = I(\bX_i \in \df')$, where $I(\cdot)$ is an indicator function. We then want to learn $P(Z = 1 |\bX)$, i.e. the probability that an observation came from one dataset or another. Note that this relies on the density discrepancy between $P_1^*$ and $P_2^*$, which may be a nonlinear function of complex interactions between each feature. Then, it follows that
\[
P(Z_i = 1 | \bX_i) = \frac{P(\bX_i | Z_i =1)P(Z_i = 1)}{P(\bX_i)} = \frac{dP_2^*(\bX_i)P(Z_i = 1)}{P(\bX_i)} 
\]
and thus
\[
\frac{P(Z_i = 1 | \bX_i)}{P(Z_i = 0 | \bX_i)} = \frac{\frac{dP_2^*(\bX_i)P(Z_i = 1)}{P(\bX_i)} }{\frac{dP_1^*(\bX_i)P(Z_i = 0)}{P(\bX_i)} } = \ell(\bX_i) \frac{P(Z_i = 1)}{P(Z_i = 0)}.
\]
We only require our importance sampling weights to be proportional to $\ell(\bX_i)$, so that any information placed in the marginal distribution of $Z_i$ is accounted for in the normalization. Using the random forest estimates $\hat{\pi}_i$ of $\pi_i = P(Z_i = 1 | \bX_i)$, we let $w(\bX_i) := w_i = \frac{\hat{\pi}_i}{1 - \hat{\pi}_i}$ be our estimate of the appropriate weighting scheme. To ameliorate dividing by 0, in practice, we add a small constant $\delta$ to both the numerator and denominator. 
We now provide an approximate error estimate of the classifier-inverted ratio weight. We can write $\hat{\pi}_i = \pi_i + \epsilon_i$ for some error term $\epsilon_i$ which we assume has finite variance $\sigma^2_\epsilon$. Then, the ratio weights are given by
\[
w_i = \frac{\hat{\pi}_i }{1 - \hat{\pi}_i} = \frac{\pi_i + \epsilon_i}{1 - \pi_i - \epsilon_i} := g_i(\epsilon_i)
\]
where $g_i$ is a differentiable function with derivative $g'_i(x) = (1 - \pi_i - x)^{-2}$. Then, assuming that $\epsilon_i$ satisfies both a central limit theorem and a law of large numbers (asymptotic in $N$) we see that
\[
\text{Var}\left(\sqrt{N} w_i\right) \approx g'\left(\E \epsilon_i\right)^2 \sigma^2_\epsilon = \frac{\sigma^2_\epsilon}{(1 - \pi_i - \E \epsilon_i)^4}
\]
so that if the asymptotic bias ($\E \epsilon_i$) is small or 0, the variance of the weight estimates scales as $O\left((1-\pi_i)^{-4}\right)$. This can lead to severe instability in the probabilistic classifier estimate, if the underlying conditional probabilities are close to 1. The effect of this instability is shown in \autoref{fig:DR_comparison}, where even in a simple univariate case, the probabilistic classifier picks up on the general trend of the density ratios, but has high variance. As such, an alternative method of estimating the likelihood ratio weights is needed. 

\begin{figure}[t]
    \centering
    \includegraphics[width = .75\textwidth]{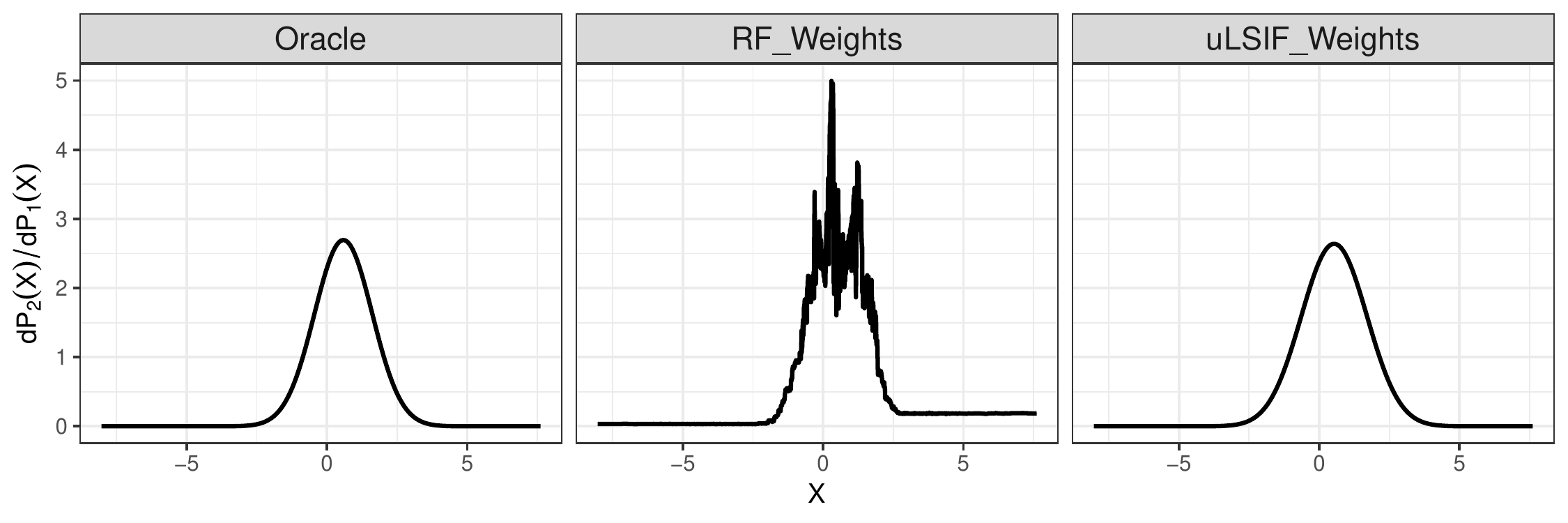}
    \caption{Comparison of estimated density ratios between an inverted random forest classifier and the uLSIF method of \citet{Kanamori2009}. In this example, $P_1^*(X) = \mathcal{N}(0, 2.5^2)$ and  $P_2^*(X) = \mathcal{N}(0.5, 0.95^2)$, and models were learned with $n = 1500$ examples from each. In the above example, the RF attained RMSE of $0.355$ while the uLSIF method attained an RMSE of $0.139$.}
    \label{fig:DR_comparison}
\end{figure}

\subsubsection{Least Squares Importance Fitting}
Another method for estimating density ratios that has been explored is Least Squares Importance Fitting, developed by \citet{Kanamori2009}. The approach essentially reduces down to modelling the ratio as a linear output
\[
\ell(\bX) = \sum_{k=1}^b \alpha_k K_\sigma(\bX, \bx_k)
\]
where $\alpha_k \geq 0$ for all $k$, $\bx_k$ are centroid points, $\sigma$ is a bandwidth parameter, and $K_\sigma(\cdot, \cdot)$ is a Gaussian kernel. The authors recommend using the points in $\df'$ as the centroids. The model fitting proceeds by minimizing the objective function
\begin{equation} \label{eqn:restr}
L_\lambda(\alpha) = \left[\frac{1}{2n}\alpha^T \left[ \sum_{i=1}^nK_\sigma(\bX_i, \bx_k)K_\sigma(\bX_i, \bx_j)\right]_{k,j = 1}^{k,j = n}\alpha - \left[ \frac{1}{m}\sum_{k=1}^{m} K_\sigma(\bX_i, \bX_k)\right]_{i=1,...,n}^T\alpha + \lambda ||\alpha||_1  \right]
\end{equation}
where $\lambda$ is a tuning parameter, and the first term uses observations from the training data, while the second term uses observations from the test data. The tuning parameters ($\sigma, \lambda$), are selected by leave one out cross validation, whose analytic form is provided by \citet{Kanamori2009}. Minimizing \autoref{eqn:restr} subject to $\hat{\alpha}_k \geq 0$ for all $k$ can be computationally expensive, so in practice, \citet{Kanamori2009} recommends using an unconstrained approximation which is provably close to the constrained estimates. Then, ratio estimates are made by calculating $w(\bX) = \sum_{k=1}^m \hat{\alpha}_k K_\sigma(\bX, \bx_k)$. This approach inherits many of the favorable properties of  regularized least squares models, and is computationally efficient. The efficacy of this model is demonstrated in the rightmost panel of \autoref{fig:DR_comparison}, where the learned density ratio is near identical to the oracle weights.

\subsubsection{Weight Regularization}
In practice, $p$ is large in many problems. Thus, the weights are likely to be either quite small or quite large, inappropriately concentrating mass on only a few points. A typical quantifier of this effect is \textit{effective} sample size, which is defined as
\[
n_{\text{eff}} = \frac{\left(\sum_{i=1}^n w(\bX_i) \right)^2}{\sum_{i=1}^n w(\bX_i)^2}.
\]
To understand effective sample size, it is useful to look at the two extreme scenarios: 1) If all the weights are uniform, then $n_{\text{eff}} = n$ and 2) if the weights are a 1-hot vector, i.e. all weights are 0 except for a single entry, then $n_{\text{eff}} = 1$. Thus, the more evenly distributed the weights, the higher $n_{\text{eff}}$, so that effective sample size is an estimate of the equivalent sample size if all the data came from $P_2$. 

Under large magnitude covariate shifts, the relative influence of certain points in the training set can grow, meaning a low effective sample size and model instability \citep{Shimodaira2000}. To combat this, a common technique is to introduce a smoothing parameter $\lambda \in (0, 1]$, and to use weights $w(\bX; \lambda) = w(\bX)^\lambda$, which has the effect of shrinking all the weights, but shrinking the large weights more severely \citep{Sugiyama2007}. Selecting $\lambda$ via typical procedures such as cross-validation is challenging, because such procedures suffer from the same flaws illustrated in \autoref{sec:intro}. As such, we instead suggest the following heuristic. First, fix $n_0 \in (1, n)$, typically as a fraction of the overall sample size. Then, select $\lambda$ such that $n_\text{eff} = n_0$ when using weights $w(\bX)^\lambda$. This is equivalent to finding the roots of 
\[
f(\lambda) = \frac{\left(\sum_{i=1}^n w(\bX_i)^\lambda \right)^2}{\sum_{i=1}^n w(\bX_i)^{2\lambda}} - n_0
\]
which can be calculated quickly in many software packages. In works such as \citet{Sugiyama2007}, the authors recommend using importance weighted cross validation to select $\lambda$. However, this weighted cross validation is calculated only with respect to $\lambda = 1$, so that the cross validation estimate may inherit some of the undesirable properties of non-regularized weights, e.g. instability and high variance. As such, we suggest \textit{a priori} selection of $\lambda$, which is then used in estimation of both the weighted random forest and the weighted model. 
\subsection{Tuning the model} \label{subsec:tuning}
A key part of any predictive analysis is estimation of generalization error. Typically, this is done through methods such as repeated training/test splits, cross validation, or bootstrapping. These procedures repeatedly use uniform resampling to create training/test splits, and loss is calculated by making predictions on the held out set using a model trained on the training split. The hyper-parameters associated with the optimal score are then recorded, and a final model is trained with those parameters. This framework is appropriate when the test set and training set are assumed to have come from the same distribution - a random sample from the empirical distribution is an unbiased approximation to a random sample from the population. The same is not true under covariate shift, but we would still like a method of tuning a model, with the goal of minimizing the generalization error under $P_2$, as in \citet{Sugiyama2007}.

Random forests (and other bagging methods) have an additional means of estimation of the generalization error: the out-of-bag (oob) error. Each base learner is trained on only a fraction of the unique instances in the training set, creating a natural training/test split. For each split, the oob error is usually calculated as the mean squared error on the held out set, and the overall oob error is given averaging across resamples. \citet{Friedman2001b} note that the oob error can be reformulated as the error associated with taking each observation $(\bX_i, Y_i)$ and constructing a random forest using only trees in which $(\bX_i, Y_i)$ did not appear in the sample, and then recording the loss when making a prediction at $\bX_i$ using this forest. Let $B_i =\sum_{j=1}^B I(\bX_i \notin \df^*_j)$, i.e. the number of resamples that do not contain $(\bX_i, Y_i)$, so that we can write the oob error as
\begin{equation}\label{eqn:oob2}
    \text{OOB}_{m,B} = \frac{1}{n}\sum_{i=1}^n \bigg(\frac{1}{B_i}  \sum_{k=1}^{B} T(\bX_i; \xi_k)I(\bX_i \notin \df^*_k) - Y_i\bigg)^2.
\end{equation}
Because $\lim_{B\to\infty} B_i = \infty$, we can construct an infinite random forest for each point, so that by the law of large numbers, $\lim_{B\to\infty} \text{OOB}_{m,B} = \frac{1}{n}\sum_{i=1}^n \left( \E_\xi T(\bX_i; \xi, \df_{-i}) - Y_i\right)^2$. Thus, as $B\to\infty$, \autoref{eqn:oob2} approaches the $n$-fold cross validation error, which is then used as an estimate of the generalization error of the forest. Similarly, we define the weighted oob error as
\begin{equation}\label{eqn:oob3}
    \text{OOB}^{\bm{w}}_{m,B} = \frac{1}{\sum_{j=1}^n w_j}\sum_{i=1}^n w_i\bigg(\frac{1}{B_i}  \sum_{k=1}^{B} T_{\bm{w}}(\bX_i; \xi_k)I(\bX_i \notin \df^*_k) - Y_i\bigg)^2.
\end{equation}
In what follows, we let $m_{B_i}(\bX_i) = \frac{1}{B_i}  \sum_{k=1}^{B} T_{\bm{w}}(\bX_i; \xi_k)I(\bX_i \notin \df^*_k)$ be the random forest trained using only trees that did not see observation $(\bX_i, Y_i)$. The utility of this weighted metric is a result of the following proposition.
\begin{prop} \label{prop:prop1}
Let $\{Z_i\}_{i=1}^N \iid Bernoulli(\alpha)$, and let $ (\bX_i, Y_i)_{i=1}^{n+m} | Z_i \iid Z_i P_2 + (1-Z_i)P_1$, where $P_1$ and $P_2$ satisfy \autoref{eqn:dist_fact}. Define $m = \sum_{i=1}^N Z_i$. Assume that $Y_i \geq 0$ almost surely, $\sup_{\bx}\E(Y^4 | \bX = \bx) < K$ for some constant $K$, and that \[
\rho^*_n = \max_{k = 1,2}\max_{i\neq j} Cor_{P_k}\bigg[(m_{B_i}(\bX_i) - Y_i)^2, (m_{B_j}(\bX_j) - Y_j)^2\bigg] \to 0
\]
as $n\to\infty$. Further, assume that for all $\bx \in \mathcal{X}$, $w_N(\bx)$ is consistently proportional to the likelihood ratio, $\ell(\bx) = \frac{dP^*_2(\bx)}{dP^*_1(\bx)}$, so that $w_N$ satisfies
\[
w_N(\bx) = c \frac{dP^*_2(\bx)}{dP^*_1(\bx)} + \epsilon_N(\bx)  \ \forall \ \bx \in \mathcal{X}
\]
where $c$ is a constant that does not depend on $\bx$, and $\epsilon_N(\bx)$ is a sequence of random variables satisfying $P(\sup_{\bx}|\epsilon_N(\bx)| < \eta_N) = 1$, where $\eta_N \to 0$ as $N\to\infty$. Let $\theta_{P_2} = \E_{P_2}(\lim_{B\to\infty}\text{OOB}_{m,B})$. Then, as $B, n \to\infty$
\[
\text{OOB}^{\bm{w}}_{m,B} \stackrel{p}{\to} \theta_{P_2}.
\]
\end{prop}
\noindent\citet{Sugiyama2007} showed that the weighted $n$-fold CV is \textit{almost} unbiased for the true validation error under $P_2$, so that often $\theta_{P_2} = \E_{(\bX, Y) \sim P_2} (m_B(\bX) - Y)^2$. The upshot of this result is that we can use the weighted oob error as a consistent metric of the generalization error for data from $P_2$, and so minimizing the weighted oob error in training should produce a good model for data from $P_2$.  

\subsection{Dealing with missing data}
\label{subsec:missing}
A challenge of using a dataset agglomerated from many diverse sources are missing observations. Discarding missing observations is not desirable, but imputation should be done in a careful manner. In particular, because the procedure above relies on the training data all coming from one distribution, standard imputation procedures (such as mean imputation) effectively impose a new distribution on the missing covariates. To overcome this, we propose the following iterative procedure:
\begin{enumerate}
    \item Let $\mathcal{M}_0 \subset \{1,...,p \}$ denote the column indices of covariates with missing observations, and let $\bX_{\mathcal{M}_0} = \{X^{(j)} : \ j \in \mathcal{M}_0\}$, and similarly let $\bX_{\mathcal-{M}_0} = \{X^{(j)} : \ j \notin \mathcal{M}_0\}$ 
    \item Sample a covariate $X^{(j)}$ from the columns of $\bX_{\mathcal{M}_0}$ randomly. Train a random forest with $X^{(j)}$ as the response, using only data from $\bX_{-\mathcal{M}_0}$. This requires subsetting the dataset to $\{ \bX_i : X_i^{(j)} \ \text{is not missing}\}$.
    \item For each $\{\bX_i : X_i^{(j)} \ \text{is missing}\}$ sample $U_i \sim Unif(0, 1)$ and set $X_i^{(j)} = \hat{Q}_{U_i}(\bX_{i, -\mathcal{M}_0})$. Set $\mathcal{M}_1 = \mathcal{M}_0 \setminus \{j\}$.
    \item Repeat steps (2)-(3), at each stage sampling covariate $j_k$ from $\mathcal{M}_k$ to serve as the response, where $\mathcal{M}_k = \mathcal{M}_{k-1}\setminus\{j_k\}$ for $k = 1,..., |\mathcal{M}_0|$. 
\end{enumerate}
This procedure is, at first glance, similar to the \texttt{missForest} procedured proposed by \citet{Stekhoven2011}, who use a standard regression/classification forest to impute the missing values, which is a form of conditional mean imputation, i.e. imputation of $\E(X^{(j)} | \bm{X}_{-j})$. However, a degenerate distribution at the conditional mean is not the same as the full conditional distribution of $X^{(j)} | \bm{X}_{-j}$, and thus is incompatible with the likelihood procedure described earlier.

The process of using quantile regression for imputation is studied in \citet{Chen2014}, who studies the properties of using parametric and semi-parametric quantile regression for response imputation in a regression context. Now we make the following assumptions, which are motivated by results in \citet{Meinshausen2006}.
\begin{description}
\item[(A1) Continuous, strictly increasing CDF] Let $F_j(x|\bX_{-j} = \bx_{(-j)}) = P(X^{(j)} \leq x | \bX_{-j} = \bx_{(-j)}) $ be the conditional distribution function of each covariate. Then, we assume that $x_1 > x_0 \implies F_j(x_1) > F_j(x_0)$, and that $F_j(x)$ is continuous for every $x \in \mathbb{R}$.
\item[(A2) Access to consistent CDF estimator] Assume that $\hat{F}_j(x|\bX_{-j} = \bx_{(-j)})$ satisfies 
\[
\hat{F}_j(x|\bX_{-j} = \bx_{(-j)}) \stackrel{p}{\to} F_j(x|\bX_{-j} = \bx_{(-j)}) \text{ for all } x \in \mathbb{R}, \text{ as } n \to \infty .
\]
\end{description}
Any distribution satisfying (A1) will have a well-defined conditional quantile function, $Q_p(\bx_{-j}) = F_j^{(-1)}(p | \bX_{-j} = \bx_{-j})$; further, the conditional quantile function will be continuous. While the empirical CDF is not everywhere-continuous, we can still define $\hat{F}_j^{(-1)}(p) =  \inf\{x: \ \hat{F}_j(x) \geq p \}$. Then, (A2) implies that $ F_j(\hat{F}_j^{(-1)}(p))\stackrel{p}{\to} \hat{F}_j(\hat{F}_j^{(-1)}(p)) = p $ for all $p \in (0, 1)$. Because $F_j^{(-1)}$ is continuous, the continuous mapping theorem gives that
\begin{equation} \label{eqn:PIT}
    F_j^{(-1)}(F_j(\hat{F}_j^{(-1)}(p))) = \hat{F}_j^{(-1)}(p) \stackrel{p}{\to} F_j^{(-1)}(p)  \text{ as } n\to\infty \ \forall p \  \in (0,1).
\end{equation}
\autoref{eqn:PIT} holds uniformly for $p$ in the unit interval, so it will also hold for $U \sim Unif(0,1)$. The probability integral transform gives that $F_j^{-1}(U)$ is a random variable with CDF $F_j$. The quantile regression forests of \citet{Meinshausen2006} satisfy (\textbf{A2}) for a wide class of distributions, and so the upshot of this result is that the imputation scheme suggested above provides a consistent way of generating imputations that follow $P_1^*$. Thus, this imputation scheme is compatible, asymptotically, with the likelihood ratio procedure described earlier.

\section{Simulations}\label{sec:verify}
We now provide a variety of simulations to demonstrate the utility of our proposed method in various settings.
\subsection{An Illustrative Regression Example}
We begin with a simple example of a covariate shifted model, and demonstrate that the weighted forest can indeed pick up on local behavior. The model for the simulation is given by
\[
\begin{split}
        Y | X &\sim \mathcal{N}\left(\varphi(X), 0.5 \right) \\
        \varphi(X) &= \max\left\{\frac{e^X}{1 + e^X} \sin(X),\frac{e^{-X}}{1 + e^{-X}} \sin(-X)  \right\}
\end{split}
\]
where $\varphi(x)$ has considerable local structure. 
To simulate covariate shift, we draw training data according to $P_1(X) = \mathcal{N}(-4, 3.5^2)$ and testing data according to $P_2(X) = \mathcal{N}(3.5, 1.5^2)$. The training distribution is quite dispersed, whereas the test distribution concentrates mass around a particular region of the real line. 
We implement \autoref{alg:weightedRF} using two sources of weights: 1) Learned weights from the method of \citet{Kanamori2009} and 2) oracle weights, corresponding to $\ell(X) \propto \frac{\phi\left(\frac{X -3.5}{1.5} \right)}{\phi\left(\frac{X + 4}{3.5} \right)}$, where $\phi(\cdot)$ is the standard normal density function. We draw $n = 500$ and $n_{\text{test}} = 250$ points from the shift model as the validation set. Results are shown in \autoref{fig:CovShift_Results}. We see that the unweighted forest struggles to pick up on the main signal in the test area, while the oracle weighted and learned weighted forests come much closer to the true signal. The unweighted forest fits a constant function on the high mass regions of $P_2$, whereas the oracle/learned weight forests are much closer to the truth. Note that this improvement comes at the cost of decreased performance in the region around $X = 0$, but this area is does not contribute much mass to the RMSE under $P_2$. The learned weights are approximately correct until around $X = 3$, at which point the lack of data in this region leads to a decline in weight performance. Running this simulation over 150 runs, we see that on average the ranger model has $RMSE = 0.2440$, the learned weighted model has $RMSE = 0.1565$ and the oracle weighted model has $RMSE = 0.1133$. While model performance is more than just RMSE, we see a convincing case that the weighted forest is able to adapt to a specified region of interest.

\begin{figure}[t]
    \centering
    \includegraphics[width = .75\textwidth]{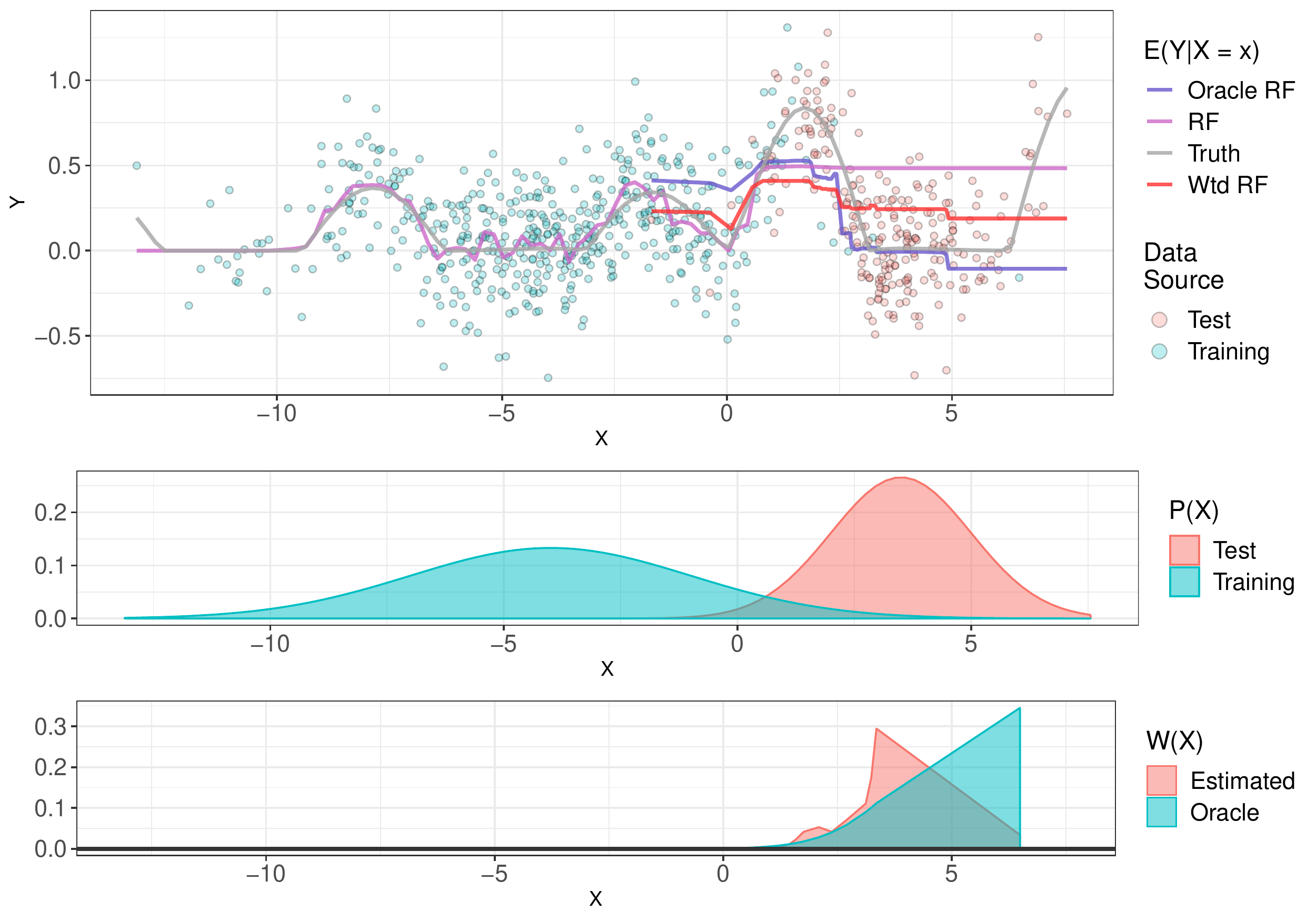}
    \caption{\textit{Top:} Fitted functions according to the three tested models, along with an overlay of the training points. \textit{Center:} The training and test densities used. \textit{Bottom:} Estimated density ratio terms and true density ratio terms.}
    \label{fig:CovShift_Results}
\end{figure}

\subsection{High Dimensional Simulation} \label{subsec:highdim}
We now compare our procedure against a baseline random forest. 
The random forest models used are trained using the \texttt{ranger} package \citep{Wright2015}. For computational efficiency, the resampling is done without replacement so that each tree is trained on $k_n < n$ unique observations. Since approximately 63\% of the dataset is represented in a given bootstrap resample, so we take $k_n = 0.6n$. Implementation of the weighted forest is done using the \texttt{rpart} package using the weights option \citep{Therneau1997}. For each model, we build $B = 500$ trees. 

We draw $150$ datasets of size $n = 1000$ with $p = 31$ covariates along with $n_{test} = 200$ points to be used as a validation set. The covariate distribution is given by
\begin{align*}
    [X^{(1)},..., X^{(6)}] &\sim \text{Dirichlet}(\bm{\alpha})\\
    X^{(7)},..., X^{(31)} &\iid \text{Uniform}(0,1).
\end{align*}
where $\bm{\alpha}$ is a pre-specified parameter. For the training set, we use $\bm{\alpha}_1 = \lambda^{[1, 2, 3, 4, 5, 6]}$ and for the test set, we use $\bm{\alpha}_2 =  \lambda^{[6, 5,4,3, 2, 1]}$, where $\lambda > 0$ is a parameter that controls how disparate the densities are (higher $\lambda$ leads to higher discrepancy). In these simulations, we use $\lambda \in \{\texttt{1, 1.07, 1.14, 1.21, 1.29, 1.36, 1.43, 1.5}\}$ - noting that $\lambda = 1$ is the case where $P_1 = P_2$.

Note that $P_2$ concentrates much more density on $X^{(5)}, X^{(6)}$ than $P_1$, but they still have the same support. The inclusion of 25 predictors whose distribution does not change is to reflect the fact that $P_1$ and $P_2$ may include the same marginal distribution for many covariates. We simulate a response, $Y$, using several different response functions, summarized in \autoref{tab:ydistn}.
\begin{table}[t]
    \centering
    \begingroup\footnotesize
    \begin{tabular}{cl}
    \hline
       Model \#  &  Data Generating Model\\
       \hline
         1 &  $Y = 5X^{(1)} + \epsilon$ \\
         2 &  $Y = 5\sin(\pi X^{(1)}) + \epsilon$ \\
         3 &  $ Y = 10\sin(\pi X^{(1)} X^{(2)}) + 20(X^{(3)} - 0.5)^2 + 10X^{(4)} + 5X^{(5)} + \epsilon$ \\
         4 & $Y = 5e^{2\sqrt{X^{(1)}X^{(2)}} + X^{(6)}} + \epsilon$ \\ 
         5 & $Y = 5\sum_{j=1}^5 \big(X^{(j)}\big)^2 + \epsilon$\\
         \hline
    \end{tabular}
    \caption{Distributions of $Y|\bm{X}$ for each model used in the simulation. In each case, $\epsilon$ is mean 0, Gaussian noise with $\E(\epsilon^2) = 0.25$.}
    \label{tab:ydistn}
    \endgroup
\end{table}
Model 1 is intended to demonstrate a situation where the marginal distribution of $Y$ may vary dramatically between $P_1$ and $P_2$. 
Model 2 shows a situation where the conditional mean is a periodic function of $X^{(1)}$, so discrepancies in the magnitude of $X^{(1)}$ should affect the response less adversely. Model 3 is the popular MARS simulation model \citep{Friedman1991}, which has been used as a stand-in for a complex regression function in previous work \citep{Mentch2016, Xu2016}. Model 4 similarly represents a complex function with a discontinuity. Finally, Model 5 represents a model where the marginal distribution of $Y$ is agnostic to changes between $P_1$ and $P_2$. 

\subsection{Simulation Results}
We analyze simulation results over both the data generating model and over the $\lambda$ parameter which controls the discrepancy in $P_1$ and $P_2$. The resulting scores (calculated according to \autoref{eqn:score}), RMSEs, and coverage probabilities are shown in \autoref{fig:highdim_results}. Tables of results are ommitted from the main text for conciseness, and instead are available in \autoref{appdx:DetailedSims}. 

In general, according to the score metric, the weighted forest performs better than the unweighted forest in Models 1 and 2. Moreover, performance is stronger in models 3 and 4 until a certain point, when the shift becomes too large. In model 5, unsurprisingly, the weighted and unweighted forest perform near identically, because the marginal distribution of $Y$ is not changing drastically. Further, looking at the RMSE plots, we see that the weighted forest is consistently able to attain a lower error rate than the unweighted forest in Models 1-4, with some breakdown at high $\lambda$. The one area where performance of the weighted model is somewhat worse than unweighted model is in coverage percentage, where the prediction intervals have slightly lower coverage in many of the situations. However, we note that the weighted procedure still maintains the nominal coverage in all cases for small values of $\lambda$. Moreover, in Models 1 and 2, the shift affects the weighted forest less severely than in Models 3 and 4. Finally, results presented in the appendix show that the weighted forest incurs much smaller prediction intervals  than those of the unweighted procedure. Thus, the weighted forest sacrifices some small coverage probability (and often does not drop below the nominal level) in exchange for much narrower prediction intervals.

\begin{figure}[p]
    \centering
    \includegraphics[width = .8\textwidth]{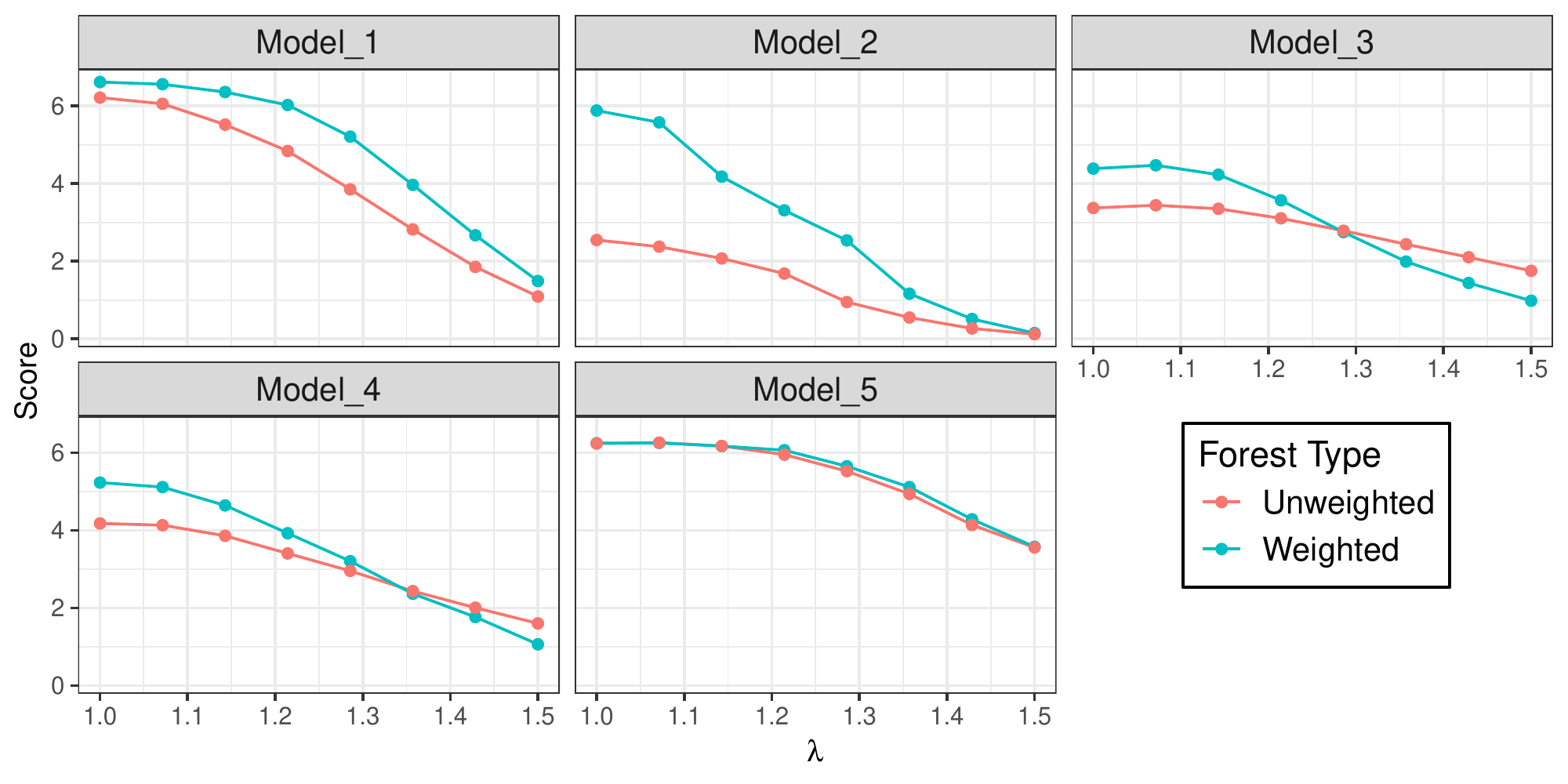}
    \includegraphics[width = .8\textwidth]{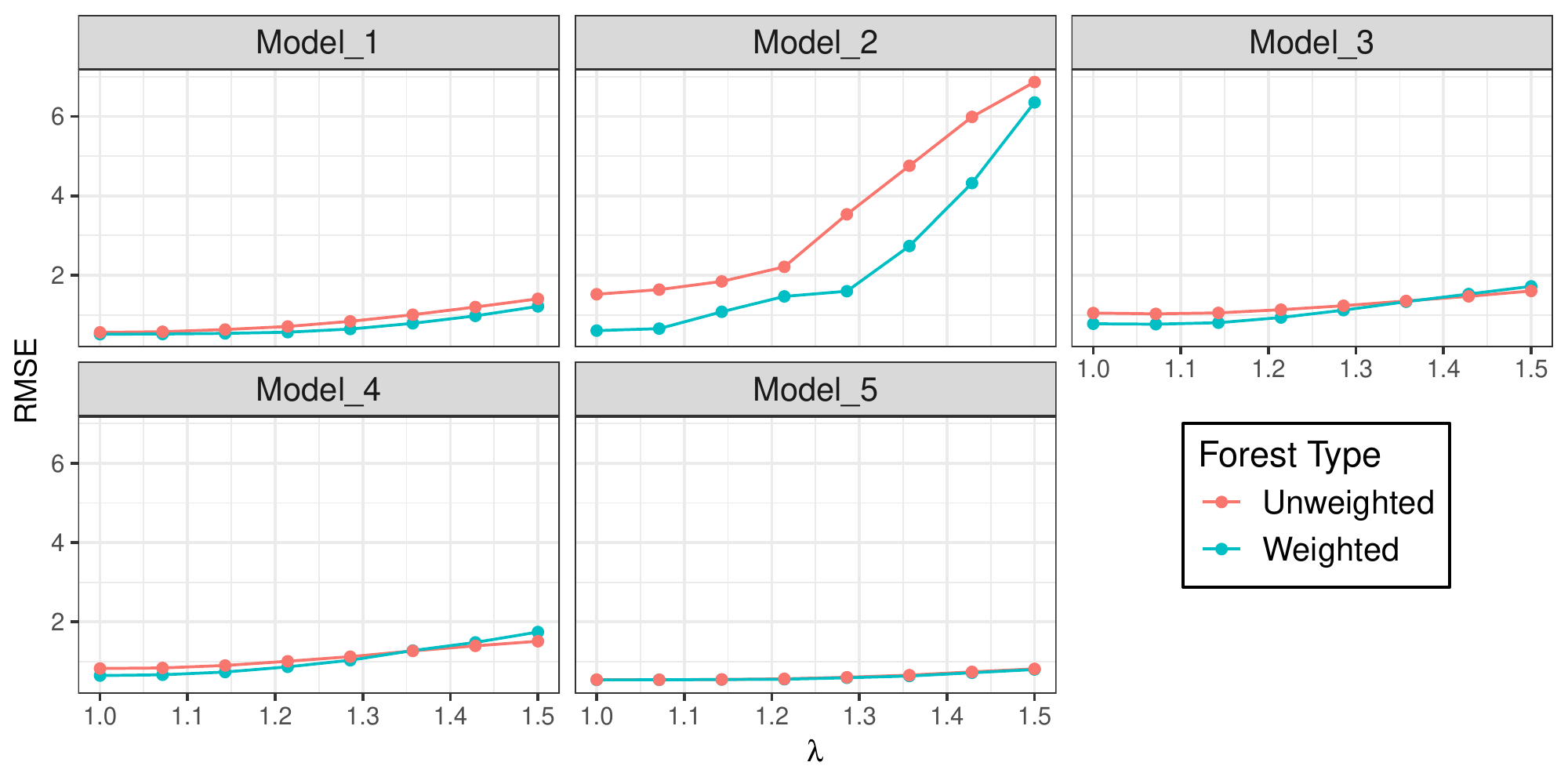}
    \includegraphics[width = .8\textwidth]{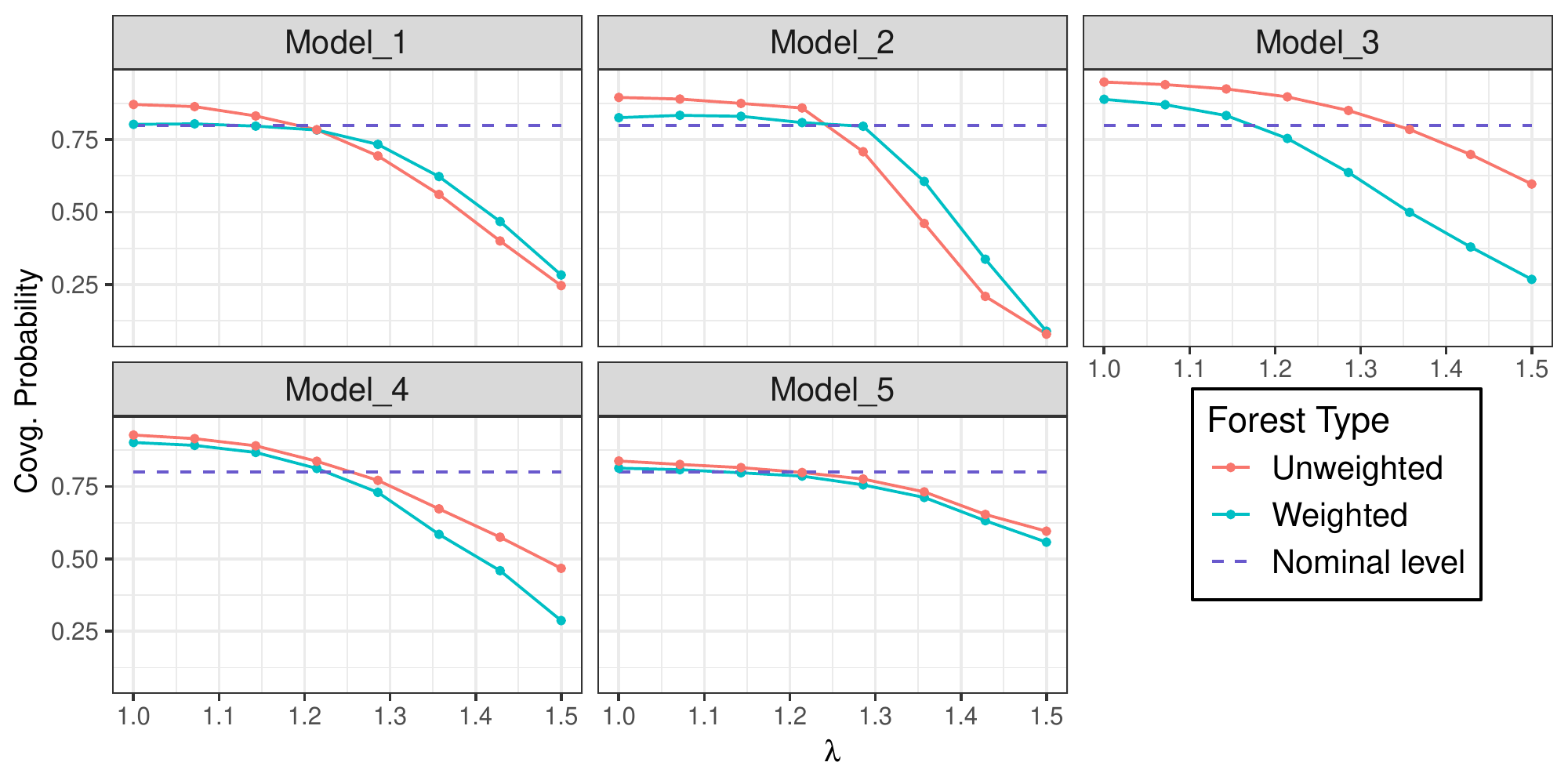}
    \caption{Results for the Score (top), RMSE (center), and Coverage probabilities (bottom) from the simulation study from \autoref{subsec:highdim}. The dashed line in the bottom indicates the nominal coverage level, $0.80$.}
    \label{fig:highdim_results}
\end{figure}

\section{Application to Hurricanes}
\label{sec:hurr}
We now turn to the problem of forecasting hurricane power outages. To begin, we apply this procedure described in \autoref{subsec:missing} to impute the missing values in the training data. In total, 26 columns had missingness and there were a total of 12244 observations that needed imputation, a non-negligible portion of the dataset. We note that because of how the training/test splits overlap from storm to storm, the imputation procedure covers both the training and test sets. We fit a weighted forest and a random forest with $\texttt{mtry} = 50$ and $\texttt{nodesize} = 5$, corresponding to the parameters suggested from \autoref{tab:unw_metrics}. For the weighted model, we again use the method of \citet{Kanamori2009} to estimate the weights. Moreover, we fix the minimum effective sample size at $n_0 = 0.75n$ and run the optimization procedure from \autoref{subsec:lrt_learn} to estimate the weight regularization $\lambda$. The results are presented in \autoref{tab:WeightedStorms}. We see that the performance in general is similar between the weighted and unweighted models, but the weighted model provides slight improvements in Harvey, Irma, and Matthew, in terms of the score metric. 

\begin{table}[t]
\centering
\begingroup\footnotesize
\begin{tabular}{c|c|c|ccccc}
  \hline
Storm & Model & $\lambda$ & RMSE & MAE & Covg & Interval Width & Score \\ 
  \hline
Harvey-2017 & Weighted & 0.1305 & 0.9097 & 0.7327 & 0.8014 & 2.5033 & 3.6171 \\ 
  Harvey-2017 & Unweighted & 0.1305 & 0.9069 & 0.7467 & 0.7679 & 2.4358 & 3.4848 \\ 
   \hline
Irma-2017 & Weighted & 0.0084 & 1.4021 & 1.1608 & 0.4615 & 2.4658 & 1.6394 \\ 
  Irma-2017 & Unweighted & 0.0084 & 1.4111 & 1.1786 & 0.3776 & 2.3777 & 1.3592 \\ 
   \hline
Sandy-2012 & Weighted & 1.0000 & 1.2286 & 1.0357 & 0.5391 & 2.1310 & 2.1901 \\ 
  Sandy-2012 & Unweighted & 1.0000 & 1.2204 & 0.9876 & 0.5521 & 2.2075 & 2.2353 \\ 
   \hline
Nate-2017 & Weighted & 0.3602 & 0.8355 & 0.7225 & 0.8528 & 2.6684 & 3.8660 \\ 
  Nate-2017 & Unweighted & 0.3602 & 0.8154 & 0.6746 & 0.8615 & 2.4930 & 4.1285 \\ 
   \hline
Matthew-2016 & Weighted & 1.0000 & 0.7932 & 0.6193 & 0.8898 & 2.5561 & 4.3897 \\ 
  Matthew-2016 & Unweighted & 1.0000 & 0.7943 & 0.6298 & 0.8924 & 2.6867 & 4.2988 \\ 
   \hline
Arthur-2014 & Weighted & 1.0000 & 1.0724 & 0.8634 & 0.6721 & 2.3111 & 2.8540 \\ 
  Arthur-2014 & Unweighted & 1.0000 & 1.0616 & 0.8432 & 0.6745 & 2.2994 & 2.8983 \\ 
   \hline
\end{tabular}
\endgroup
\caption{Model performance by storm, with weighted and unweighted storms fitted. Bolded values represent the better of the two by storm and loss function. $\lambda$ value reported is selected by the effective sample size calculation from \autoref{subsec:lrt_learn}.} 
\label{tab:WeightedStorms}
\end{table}

As a followup, we additionally implemented a study of tuning the model using the weighted out of bag metric from \autoref{subsec:tuning}. To do this, we tune the \texttt{mtry} parameter over a grid consisting of $\mathcal{M} = \{27, 39, 51, 63, 75\}$ for both an unweighted and weighted random forest. For the weighted forest, we record $\text{OOB}^{\bm{w}}_{m, B}$ and the weighted RMSE, and $\text{OOB}_{m, B}$ and the unweighted RMSE. The results are shown in  \autoref{fig:OOB_figure}, where the out of bag error for each mtry value is plotted against the RMSE of that model. For all storms except Hurricane Nate, we see that both $\text{OOB}_{m,B}$ and $\text{OOB}^{\bm{w}}_{m,B}$ dramatically underestimate the holdout RMSE, with the weighted out of bag error providing a slightly less biased estimate. However, in the context of model selection, typically the model with the lowest out of bag error (and thus lowest estimated generalization error) is selected. Thus, for model selection purposes, the generalization error estimate is less important than the ranking. We see that the weighted oob error selects an optimal model for Hurricane Matthew, and a near optimal model for hurricanes Irma and Sandy, while the unweighted model selects an optimal model for Hurricane Sandy, and a near optimal model for Irma, Nate, and Matthew. Moreover, for Hurricane Matthew, the OOB-RMSE rankings are recovered exactly, and for Hurricane Irma the same is true with the exception of one \texttt{mtry} value. In the unweighted case, there are no such clear stories. 
\begin{figure}[H]
    \centering
    \includegraphics[width = .6\textwidth]{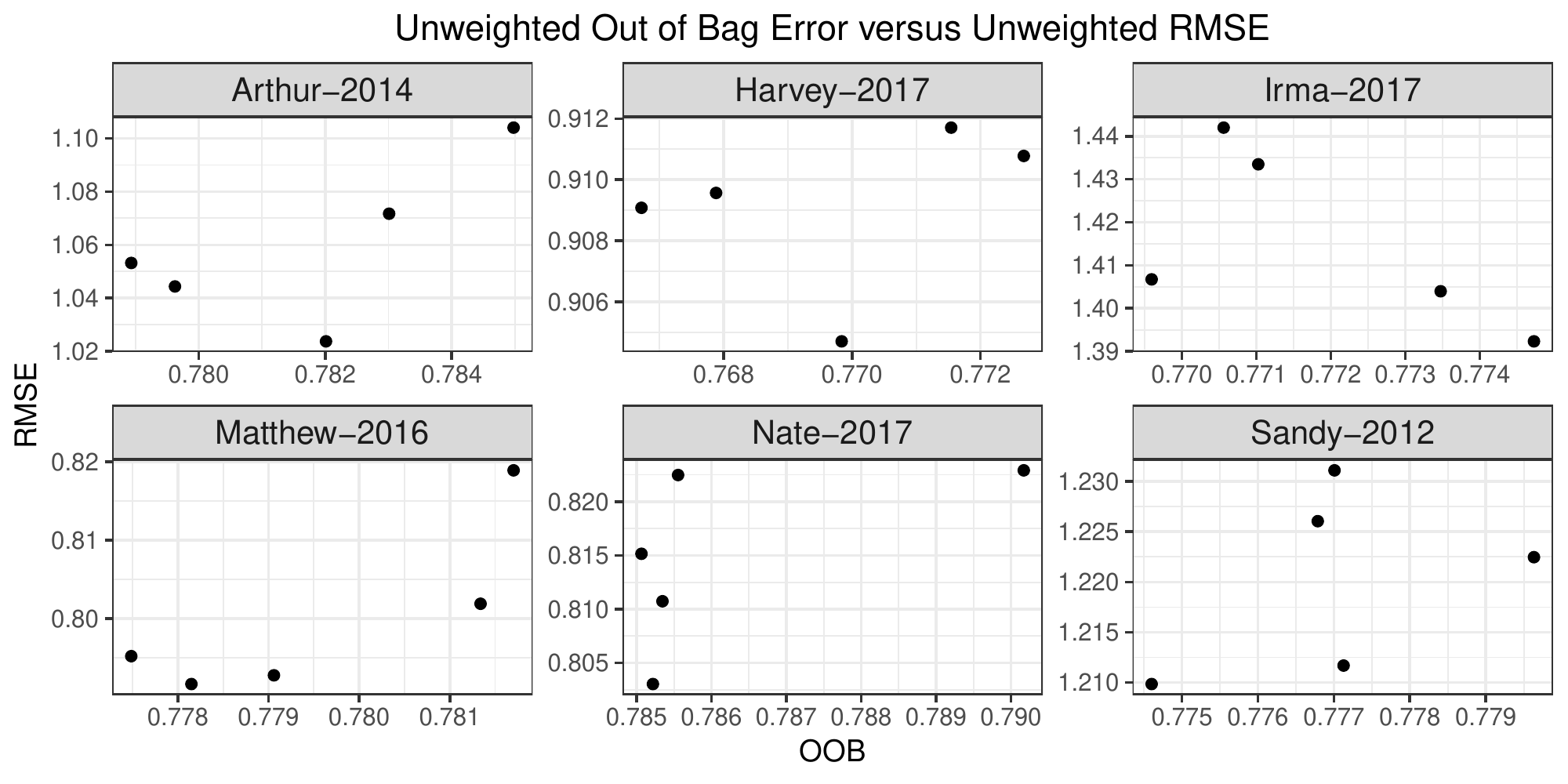}
        \includegraphics[width = .6\textwidth]{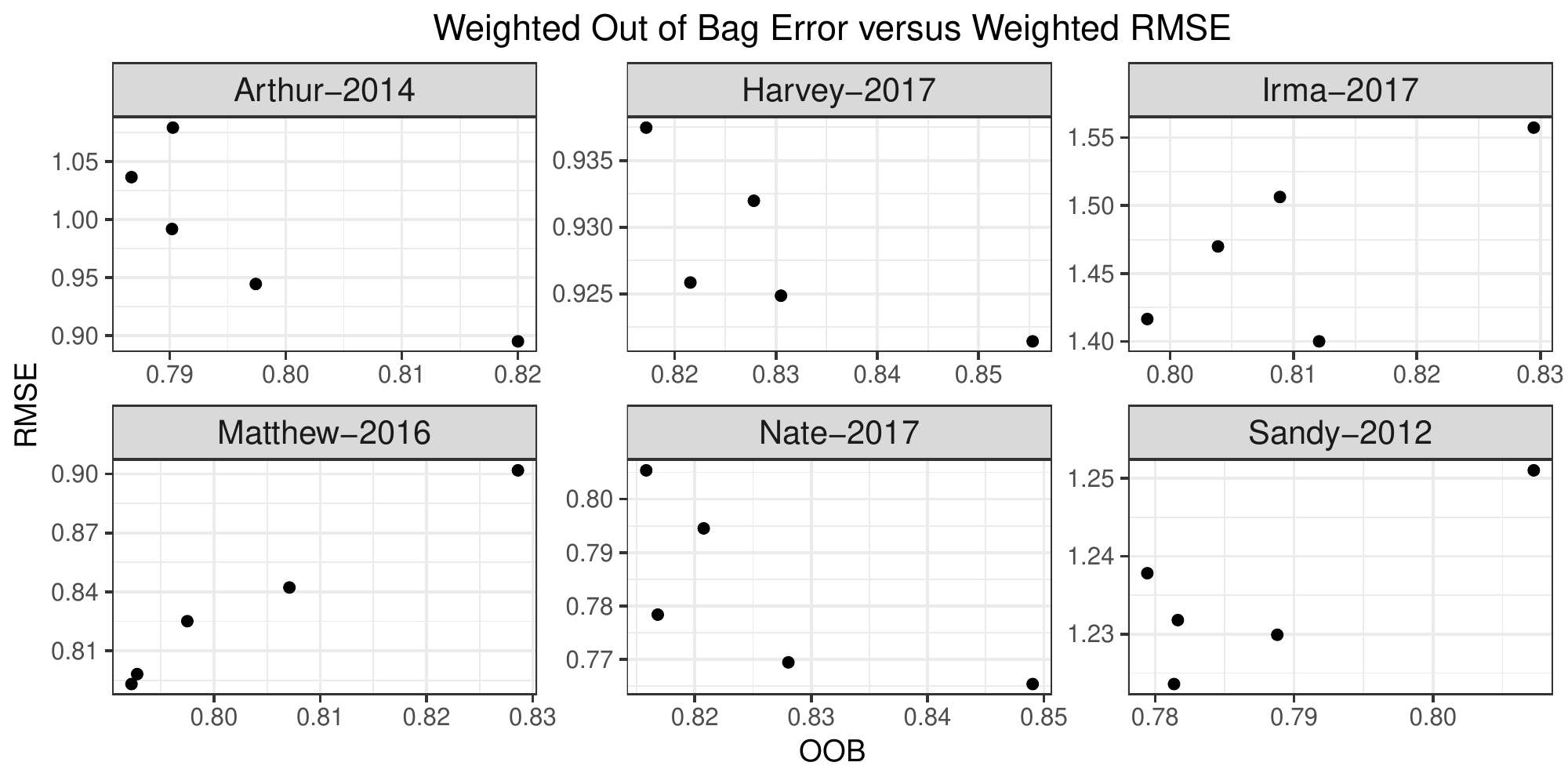}
    \caption{Out of bag error versus holdout RMSE. \textbf{Top:} Results for the unweighted forest. \textbf{Bottom:} Results for the weighted forest.}
    \label{fig:OOB_figure}
\end{figure}

\section{Conclusion}
\label{sec:conc}

We sought to modify the random forest algorithm to account for distributional changes between test and training sets, which often arise in practice. We accomplish this by imposing a covariate shift assumption, and then using existing density ratio methods to estimate the likelihood ratio weights, $\ell(\bX) \propto \frac{dP_2(\bX)}{dP_1(\bX)}$. We moreover provided methods for imputing missing data and tuning the model in ways that respect the statistical assumptions associated with the problem. The simulation study included clearly demonstrates the utility of the proposed method - the importance weighted forest typically outperforms a standard random forest in the covariate shift case. However, importance weighting is only able to address small changes in covariate distribution. Indeed, in \autoref{fig:highdim_results} it was shown that both the weighted and unweighted forest perform worse as the magnitude of the shift grows.

\bigskip
\begin{center}
{\large\bf SUPPLEMENTAL MATERIALS/ACKNOLWEDGEMENTS}
\end{center}
\begin{description}
\item[Supplementary Appendix] Appendix containing more details about the hurricane data, the method proposed, and a proof of \autoref{prop:prop1}.

\item[R File] An R file for loading the data and methods, and running the simulations.

\item[Funding] LM was partially supported by NSF DMS-1712041. This research was supported in part by the University of Pittsburgh Center for Research Computing. The research presented in this work was also supported by the U.S. Department of Homeland Security’s National Risk Management Center under the National Infrastructure Simulation and Analysis Center (NISAC) project. 

\end{description}

\newpage
\appendix
\setcounter{table}{0}
\renewcommand{\thetable}{A\arabic{table}}
\setcounter{figure}{0}
\renewcommand{\thefigure}{A\arabic{figure}}

\section{Details on the Hurricane Data} \label{appdx:DetailedTab}
Here, we provide more detail about the hurricane data used in \autoref{sec:intro} and \autoref{sec:hurr}. In particular, in \autoref{tab:covariates} we summarise the information used in the model.

\begin{table}[H]
\begin{center}
  \begin{threeparttable}[t]
 \begin{tabular}{l|l}
        \toprule
        \textbf{Predictor} & \textbf{Source}  \\
         \midrule
         Maximum Sustained Wind & NOAA-Hurdat2 \\
         Maximum Gust Wind & Estimated from Maximum Sustained Wind \tnote{1}\\
         Gust Wind Duration & Estimated from Maximum Sustained Wind \tnote{1} \\
         Population density & SEDAC 2010 \\
         Tree Species & GECSC \\
         Soil Texture & Polaris \\
         Land Cover & NLCD2011 \\
         Elevation & DEM-GMTED \\
         Soil Moisture & NOAA-CPC \\
         SPI & NOAA-NCDC \tnote{2} \\
         \bottomrule
    \end{tabular}
    \begin{tablenotes} \footnotesize{
     \item[1] Model used is part of the \texttt{R} package \texttt{hurricaneexposure} 
     \item[2] SPI refers to the standard precipitation index, and is a measure of precipitation normalized to historical records.
     }
   \end{tablenotes}
    \end{threeparttable}%
       \end{center}
    \caption{Model covariates and sources. Covariates come in different resolutions, but are aggregated to the county level}
    \label{tab:covariates}
\end{table}%

\section{Proof of Proposition 1} \label{appdx:PropProof}
Here, we prove Proposition 1, which is restated below followed by its proof.
\begin{prop}
Let $\{Z_i\}_{i=1}^N \iid Bernoulli(\alpha)$, and let $ (\bX_i, Y_i)_{i=1}^{n+m} | Z_i \iid Z_i P_2 + (1-Z_i)P_1$, where $P_1$ and $P_2$ satisfy \autoref{eqn:dist_fact}. Define $m = \sum_{i=1}^N Z_i$. Assume that $Y_i \geq 0$ almost surely, $\sup_{\bx}\E(Y^4 | \bX = \bx) < K$ for some constant $K$, and that \[
\rho^*_n = \max_{k = 1,2}\max_{i\neq j} Cor_{P_k}\bigg[(m_{B_i}(\bX_i) - Y_i)^2, (m_{B_j}(\bX_j) - Y_j)^2\bigg] \to 0
\]
as $n\to\infty$. Further, assume that for all $\bx \in \mathcal{X}$, $w_N(\bx)$ is consistently proportional to the likelihood ratio, $\ell(\bx) = \frac{dP^*_2(\bx)}{dP^*_1(\bx)}$, i.e. $w_N$ satisfies
\[
w_N(\bx) = c \frac{dP^*_2(\bx)}{dP^*_1(\bx)} + \epsilon_N(\bx)  \ \forall \ \bx \in \mathcal{X}
\]
where $c$ is a constant that does not depend on $\bx$, and $\epsilon_N(\bx)$ is a sequence of random variables satisfying $P(\sup_{\bx}|\epsilon_N(\bx)| < \eta_N) = 1$, where $\eta_N \to 0$ as $N\to\infty$. Let $\theta_{P_2} = \E_{P_2}(\lim_{B\to\infty}\text{OOB}_{m,B})$. Then, as $B, n \to\infty$
\[
\text{OOB}^{\bm{w}}_{m,B} \stackrel{p}{\to} \theta_{P_2}.
\]
\end{prop}
\begin{proof}
To show this, we use a standard trick in the importance sampling literature to rewrite $\text{OOB}^{\bm{w}}_{m,B}$ as
\begin{equation}\label{eqn:oob}
    \text{OOB}^{\bm{w}}_{m,B} = \frac{\sum_{i=1}^n w_i(m_{B_i}(\bX_i)- Y_i)^2}{\sum_{j=1}^n w_j} = \frac{\frac{1}{n}\sum_{i=1}^n w_i(m_{B_i}(\bX_i) - Y_i)^2}{\frac{1}{n}\sum_{j=1}^n w_j}\ .
\end{equation}
An important point of clarification is that we use $N$ to be the total sample size, $n$ to be the size of the training set, and $m$ be the size of the test set. Because $n \sim \text{Binomial}(N, \alpha)$, $\lim_{N\to\infty} n = \infty$ (and similarly for $m$) almost surely. Thus, we use $n\to\infty$, $m\to\infty$, and $N\to\infty$ interchangeably. The weak law of large numbers gives that as $n\to\infty$, the denominator of \autoref{eqn:oob} obeys
\[
\frac{1}{n}\sum_{j=1}^n w_j \stackrel{p}{\to} c\E_{\bX\sim P^*_1}\bigg[\frac{dP^*_2(\bX)}{dP^*_1(\bX)}\bigg] = c\int_{\mathcal{X}} dP_2^*(\bx) = c.
\]
By assumption, $w_i = c\frac{dP^*_2(\bX_i)}{dP^*_1(\bX_i)} + \epsilon_N(\bX_i)$, so that the numerator of \autoref{eqn:oob} can be expressed as
\[
\frac{1}{n}\sum_{i=1}^n \bigg[c\frac{dP^*_2(\bX_i)}{dP^*_1(\bX_i)}+ \epsilon_N(\bX_i)\bigg] \bigg(\frac{1}{B_i} \sum_{k=1}^{B_i} T_{\bm{w}}(\bX_i; \xi_k) - Y_i\bigg)^2.  
\] 
Now, we want to show that this converges in probability to $c\theta_{P_2}$. We do this by analyzing the variance of the numerator of \autoref{eqn:oob}.  Note that we have
\begin{align*}
    &\V\bigg[\frac{1}{n}\sum_{i=1}^n \bigg(c\frac{dP^*_2(\bX_i)}{dP^*_1(\bX_i)}+ \epsilon_N(\bX_i)\bigg) \bigg(\frac{1}{B_i}  \sum_{k=1}^{B_i} T_{\bm{w}}(\bX_i; \xi_k) - Y_i\bigg)^2 \bigg] \\
    &= \V\bigg[\underbrace{c\frac{1}{n}\sum_{i=1}^n \bigg(\frac{dP^*_2(\bX_i)}{dP^*_1(\bX_i)}\bigg)\bigg(\frac{1}{B_i}  \sum_{k=1}^{B_i} T_{\bm{w}}(\bX_i; \xi_k) - Y_i\bigg)^2}_{S_{1,n}} + 
    \underbrace{ \frac{1}{n}\sum_{i=1}^n \epsilon_N(\bX_i)\bigg(\frac{1}{B_i}  \sum_{k=1}^{B_i} T_{\bm{w}}(\bX_i; \xi_k) - Y_i\bigg)^2}_{S_{2,N}} \bigg]. \\
\end{align*}

\vspace{-8mm}

\noindent We approximate $\V(S_{1,n}+ S_{2,N})$ as $\V(S_{1,n}) + \V(S_{2,N})$, because $\Cov(S_{1,n}, S_{2,N}) \to 0$ as $N\to\infty$. To see this last fact, note that $S_{2,N}$ satisfies
\begin{equation}\label{eqn:s2bound}
|S_{2,N}| < \frac{\eta_N}{n}\sum_{i=1}^n\bigg(\frac{1}{B_i}  \sum_{k=1}^{B_i} T_{\bm{w}}(\bX_i; \xi_k) - Y_i\bigg)^2
\end{equation}
and that the quantity on the right hand side is integrable, so that by dominated convergence, $\E(S_{2,N}) \to 0 $. Moreover, by assumption, the squared out of bag residuals are bounded in probability (because they are assumed to have finite mean/variance). Thus, the cross-term can be controlled as
\begin{align*}
\E&\bigg[S_{2,N}\times\frac{c}{n}\sum_{i=1}^n \bigg(\frac{dP^*_2(\bX_i)}{dP^*_1(\bX_i)}\bigg)\bigg(\frac{1}{B_i}  \sum_{k=1}^{B_i} T_{\bm{w}}(\bX_i; \xi_k) - Y_i\bigg)^2\bigg]  \\
&< \E\bigg[ \frac{\eta_N}{n}\sum_{i=1}^n\bigg(\frac{1}{B_i}  \sum_{k=1}^{B_i} T_{\bm{w}}(\bX_i; \xi_k) - Y_i\bigg)^2 \times \frac{c}{n}\sum_{i=1}^n \bigg(\frac{dP^*_2(\bX_i)}{dP^*_1(\bX_i)}\bigg)\bigg(\frac{1}{B_i}\sum_{k=1}^{B_i} T_{\bm{w}}(\bX_i; \xi_k) - Y_i\bigg)^2\bigg]
\end{align*}
which, again by dominated convergence, converges to 0. 

Now, we want to show that the variance of $S_{2,N}$ vanishes. Recall that by hypothesis, $P(\lim_{N\to\infty}S_{2,N} = 0) = 1$, and so it follows that $P(\lim_{N\to\infty}S^2_{2,N} = 0) = 1$. Then, again we can appeal to dominated convergence (using the quantity in \autoref{eqn:s2bound} squared as our upper bound) to get that $\V(S_{2,N}) \to 1$ as $N \to\infty$. 
All that remains to show is that $\V(S_{1,n}) \to 0$ as $n\to\infty$. The variance of $S_{1,n}$ can be expressed as
\begin{align*}
    \V(S_{1,n}) &= \V\bigg[ \frac{c}{n} \sum_{i=1}^n \bigg(\frac{dP^*_2(\bX_i)}{dP^*_1(\bX_i)}\bigg)(m_{B_i}(\bX_i) - Y_i)^2 \bigg] \\
    &= \frac{c^2}{n^2} \sum_{i=1}^n \V\bigg[\bigg(\frac{dP^*_2(\bX_i)}{dP^*_1(\bX_i)}\bigg)(m_{B_i}(\bX_i) - Y_i)^2\bigg] \ +\\  \frac{2c^2}{n^2}\sum_{1 \leq i < j \leq n}& \Cov\bigg[\bigg(\frac{dP^*_2(\bX_i)}{dP^*_1(\bX_i)}\bigg)(m_{B_i}(\bX_i) - Y_i)^2,\bigg(\frac{dP^*_2(\bX_j)}{dP^*_1(\bX_j)}\bigg) (m_{B_j}(\bX_j) - Y_j)^2\bigg].
\end{align*}
Because $Y_i$ is almost surely positive, and $m_{B_i}(\cdot)$ is an average of positive random variables, both are positive almost surely. Also, note that the likelihood ratio term is also positive, so that the whole quantity $\big(\frac{dP^*_2(\bX_i)}{dP^*_1(\bX_i)}\big)(m_{B_i}(\bX_i) - Y_i)^2 > 0$ almost surely. Then, we make use the fact that for positive random variables $W, Z$,
\[
\V_{W,Z \sim P}\big[(W-Z)^2\big] \leq \E_{W, Z \sim P}\big[(W-Z)^4\big] = \E_{W, Z \sim Q}\bigg[\frac{dP(W,Z)}{dQ(W,Z)}(W-Z)^4\bigg]]\leq \max \big( \E_P(W^4), \E_P(Z^4)\big).
\]
Note that in the above, we use $\E_P(W^4)$ to indicate integration over the marginal distribution of $W$ under joint distribution $P$. Because $m_{B_i}$ is a weighted sum of random variables with bounded 4th moments, it also has a bounded 4th moment. Letting $\kappa = \max \{\max_{i}\E_{P_1}(m_{B_i}(\bX_i)^4), K\}$, we see that
\begin{align*}
    \V(S_{1,n}) &\leq \frac{c^2n\kappa}{n^2} + \frac{2c^2}{n^2}\sum_{1 \leq i < j \leq n} \Cov_{P_1}\bigg[\bigg(\frac{dP^*_2(\bX_i)}{dP^*_1(\bX_i)}\bigg)(m_{B_i}(\bX_i) - Y_i)^2,\bigg(\frac{dP^*_2(\bX_j)}{dP^*_1(\bX_j)}\bigg) (m_{B_j}(\bX_j) - Y_j)^2\bigg] \\
    &\leq \frac{\kappa c^2}{n} + \frac{2c^2}{n^2}n^2\kappa\rho^*_n \\
    &= \frac{\kappa c^2}{n} + 2\kappa c^2 \rho^*_n.
\end{align*}
The above goes to 0 by hypothesis, and noting that $\E S_{1,n} = c\theta_{P_2}$, we can apply Chebyshev's inequality to conclude that
\[
\frac{1}{n}\sum_{i=1}^n w_i\bigg(\frac{1}{B_i}  \sum_{k=1}^{B_i} T_{\bm{w}}(\bX_i; \xi_k) - Y_i\bigg)^2 \stackrel{p}{\to} c\theta_{P_2} \ \text{as $N\to\infty$} .
\]
Finally, Slutsky's Lemma gives that $\text{OOB}^{\bm{w}}_{m,B} \stackrel{p}{\to}\theta_{P_2}$ as $N, B\to\infty$.
\end{proof}

\section{Detailed Simulation Results} \label{appdx:DetailedSims}

The purpose of this section of the appendix is to provide specific results for the simulation from \autoref{subsec:highdim} in the form of tables. For each model described in \autoref{tab:ydistn}, we provide the full results for each $\lambda$ value. We also provide plots similar to those from \autoref{fig:highdim_results} for the MAE and Interval Width statistics, for completeness in \autoref{fig:IW_MAE_highdim}.
\begin{figure}[H]
    \centering
    \includegraphics[width = .8\textwidth]{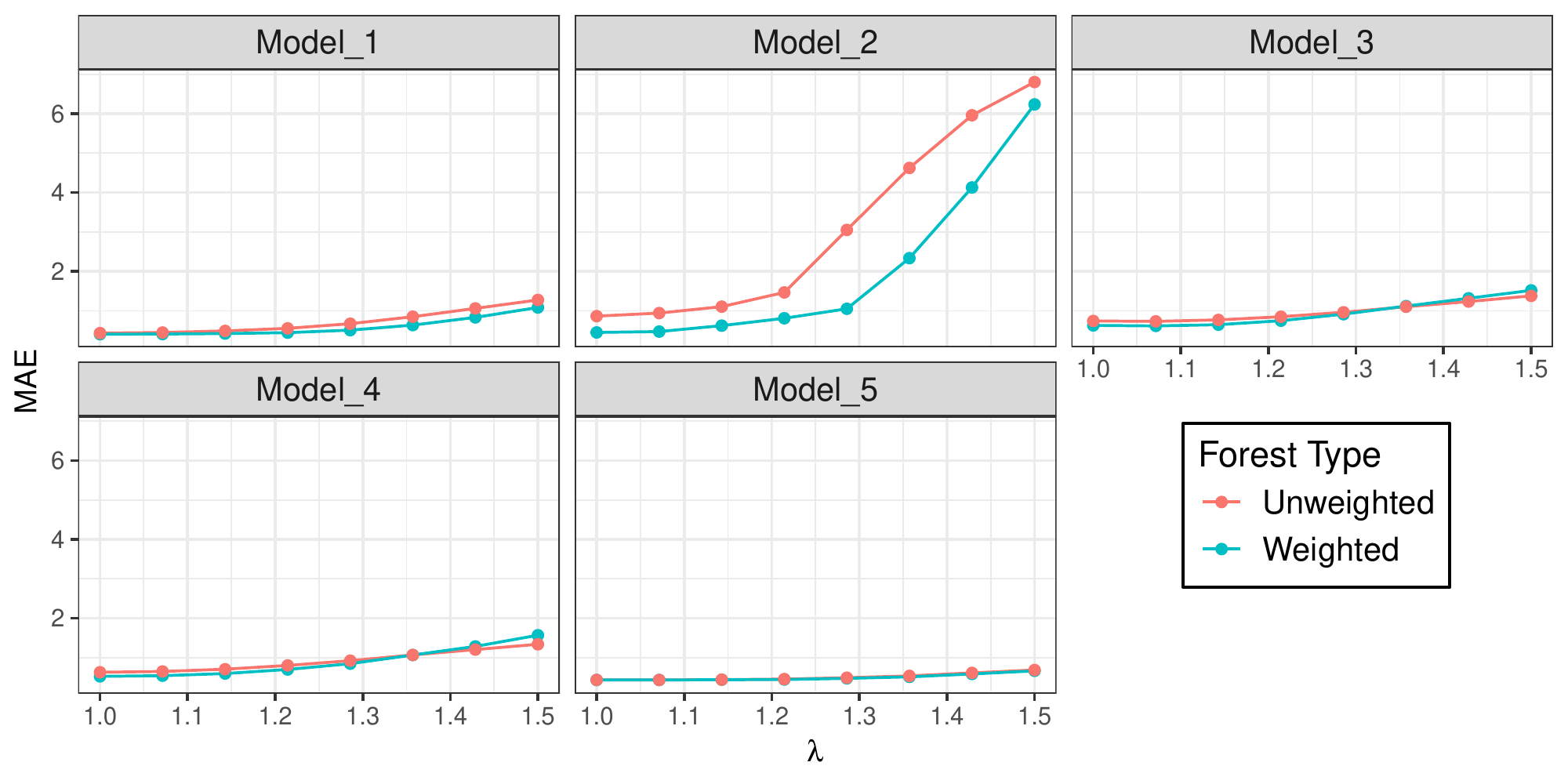}
    \includegraphics[width = .8\textwidth]{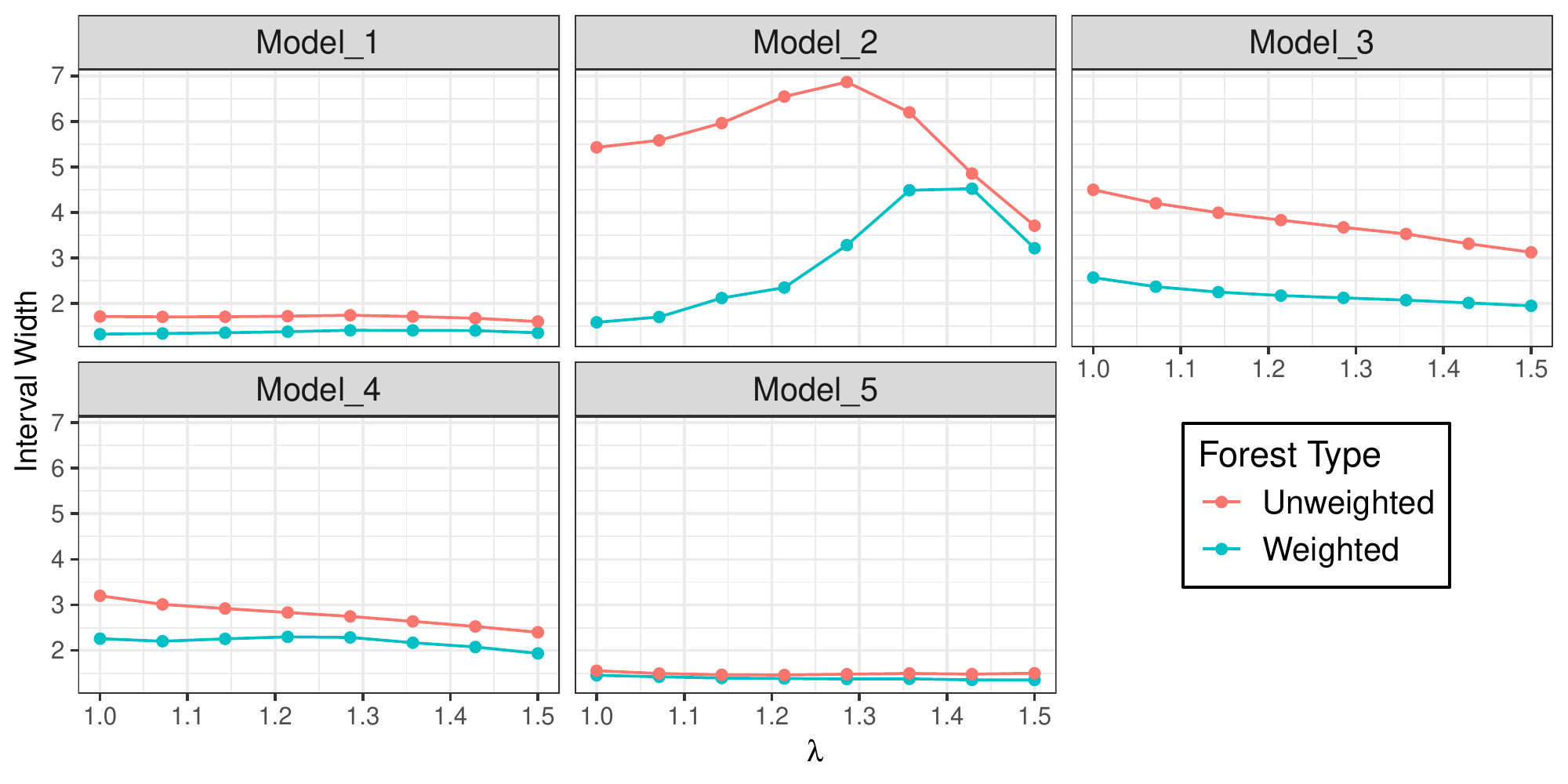}
    \caption{Results from \autoref{subsec:highdim} for MAE (top) and Interval Width (bottom)}
    \label{fig:IW_MAE_highdim}
\end{figure}

\renewcommand{\arraystretch}{.75} 
\begin{table}[H]
\centering
\begingroup\footnotesize
\begin{tabular}{cc|ccccc}
  \hline
$\lambda$ & Model & RMSE & MAE & Covg & Interval Width & Score \\ 
  \hline
1.000 & Weighted & 0.514 & 0.409 & 0.803 & 1.329 & 6.619 \\ 
  1.000 & Unweighted & 0.560 & 0.437 & 0.871 & 1.717 & 6.217 \\ 
   \hline
1.071 & Weighted & 0.518 & 0.413 & 0.805 & 1.345 & 6.562 \\ 
  1.071 & Unweighted & 0.576 & 0.451 & 0.864 & 1.710 & 6.058 \\ 
   \hline
1.143 & Weighted & 0.534 & 0.425 & 0.797 & 1.360 & 6.361 \\ 
  1.143 & Unweighted & 0.631 & 0.493 & 0.832 & 1.711 & 5.521 \\ 
   \hline
1.214 & Weighted & 0.564 & 0.448 & 0.783 & 1.384 & 6.026 \\ 
  1.214 & Unweighted & 0.708 & 0.557 & 0.784 & 1.725 & 4.840 \\ 
   \hline
1.286 & Weighted & 0.643 & 0.511 & 0.734 & 1.414 & 5.211 \\ 
  1.286 & Unweighted & 0.838 & 0.675 & 0.694 & 1.745 & 3.852 \\ 
   \hline
1.357 & Weighted & 0.787 & 0.639 & 0.623 & 1.412 & 3.968 \\ 
  1.357 & Unweighted & 1.004 & 0.852 & 0.561 & 1.717 & 2.818 \\ 
   \hline
1.429 & Weighted & 0.975 & 0.835 & 0.467 & 1.408 & 2.670 \\ 
  1.429 & Unweighted & 1.198 & 1.063 & 0.400 & 1.679 & 1.854 \\ 
   \hline
1.500 & Weighted & 1.214 & 1.087 & 0.283 & 1.360 & 1.489 \\ 
  1.500 & Unweighted & 1.404 & 1.277 & 0.246 & 1.607 & 1.090 \\ 
   \hline
\end{tabular}
\endgroup
\caption{Simulation results for Model 1. Bolded values represent the better for a given $\lambda$ setting.} 
\label{tab:Model_1}
\end{table}
\begin{table}[H]
\centering
\begingroup\footnotesize
\begin{tabular}{cc|ccccc}
  \hline
$\lambda$ & Model & RMSE & MAE & Covg & Interval Width & Score \\ 
  \hline
1.000 & Weighted & 0.604 & 0.453 & 0.826 & 1.590 & 5.884 \\ 
  1.000 & Unweighted & 1.520 & 0.867 & 0.896 & 5.430 & 2.547 \\ 
   \hline
1.071 & Weighted & 0.653 & 0.476 & 0.834 & 1.706 & 5.579 \\ 
  1.071 & Unweighted & 1.636 & 0.945 & 0.890 & 5.584 & 2.375 \\ 
   \hline
1.143 & Weighted & 1.077 & 0.626 & 0.831 & 2.123 & 4.181 \\ 
  1.143 & Unweighted & 1.843 & 1.107 & 0.875 & 5.962 & 2.071 \\ 
   \hline
1.214 & Weighted & 1.465 & 0.813 & 0.809 & 2.352 & 3.311 \\ 
  1.214 & Unweighted & 2.210 & 1.468 & 0.859 & 6.544 & 1.681 \\ 
   \hline
1.286 & Weighted & 1.596 & 1.054 & 0.796 & 3.284 & 2.536 \\ 
  1.286 & Unweighted & 3.533 & 3.053 & 0.708 & 6.864 & 0.946 \\ 
   \hline
1.357 & Weighted & 2.734 & 2.338 & 0.606 & 4.488 & 1.162 \\ 
  1.357 & Unweighted & 4.756 & 4.625 & 0.461 & 6.199 & 0.548 \\ 
   \hline
1.429 & Weighted & 4.319 & 4.129 & 0.337 & 4.523 & 0.511 \\ 
  1.429 & Unweighted & 5.987 & 5.960 & 0.209 & 4.854 & 0.266 \\ 
   \hline
1.500 & Weighted & 6.351 & 6.236 & 0.089 & 3.215 & 0.148 \\ 
  1.500 & Unweighted & 6.865 & 6.803 & 0.079 & 3.709 & 0.119 \\ 
   \hline
\end{tabular}
\endgroup
\caption{Simulation results for Model 2. Bolded values represent the better for a given $\lambda$ setting.} 
\label{tab:Model_2}
\end{table}
\begin{table}[H]
\centering
\begingroup\footnotesize
\begin{tabular}{cc|ccccc}
  \hline
$\lambda$ & Model & RMSE & MAE & Covg & Interval Width & Score \\ 
  \hline
1.000 & Weighted & 0.776 & 0.633 & 0.889 & 2.572 & 4.386 \\ 
  1.000 & Unweighted & 1.047 & 0.744 & 0.949 & 4.499 & 3.372 \\ 
   \hline
1.071 & Weighted & 0.765 & 0.619 & 0.871 & 2.372 & 4.473 \\ 
  1.071 & Unweighted & 1.026 & 0.735 & 0.940 & 4.202 & 3.442 \\ 
   \hline
1.143 & Weighted & 0.803 & 0.650 & 0.833 & 2.253 & 4.231 \\ 
  1.143 & Unweighted & 1.050 & 0.772 & 0.925 & 3.995 & 3.351 \\ 
   \hline
1.214 & Weighted & 0.934 & 0.751 & 0.754 & 2.177 & 3.570 \\ 
  1.214 & Unweighted & 1.130 & 0.852 & 0.898 & 3.833 & 3.105 \\ 
   \hline
1.286 & Weighted & 1.119 & 0.916 & 0.637 & 2.126 & 2.750 \\ 
  1.286 & Unweighted & 1.230 & 0.965 & 0.851 & 3.674 & 2.786 \\ 
   \hline
1.357 & Weighted & 1.331 & 1.122 & 0.499 & 2.076 & 1.990 \\ 
  1.357 & Unweighted & 1.351 & 1.105 & 0.785 & 3.529 & 2.437 \\ 
   \hline
1.429 & Weighted & 1.525 & 1.320 & 0.380 & 2.016 & 1.439 \\ 
  1.429 & Unweighted & 1.467 & 1.236 & 0.699 & 3.315 & 2.103 \\ 
   \hline
1.500 & Weighted & 1.719 & 1.519 & 0.268 & 1.952 & 0.982 \\ 
  1.500 & Unweighted & 1.603 & 1.381 & 0.597 & 3.127 & 1.752 \\ 
   \hline
\end{tabular}
\endgroup
\caption{Simulation results for Model 3. Bolded values represent the better for a given $\lambda$ setting.} 
\label{tab:Model_3}
\end{table}
\begin{table}[H]
\centering
\begingroup\footnotesize
\begin{tabular}{cc|ccccc}
  \hline
$\lambda$ & Model & RMSE & MAE & Covg & Interval Width & Score \\ 
  \hline
1.000 & Weighted & 0.650 & 0.527 & 0.902 & 2.258 & 5.234 \\ 
  1.000 & Unweighted & 0.829 & 0.630 & 0.927 & 3.199 & 4.180 \\ 
   \hline
1.071 & Weighted & 0.670 & 0.544 & 0.892 & 2.202 & 5.117 \\ 
  1.071 & Unweighted & 0.842 & 0.650 & 0.915 & 3.010 & 4.135 \\ 
   \hline
1.143 & Weighted & 0.740 & 0.598 & 0.867 & 2.254 & 4.644 \\ 
  1.143 & Unweighted & 0.905 & 0.708 & 0.890 & 2.920 & 3.862 \\ 
   \hline
1.214 & Weighted & 0.870 & 0.701 & 0.812 & 2.298 & 3.928 \\ 
  1.214 & Unweighted & 1.011 & 0.804 & 0.837 & 2.833 & 3.407 \\ 
   \hline
1.286 & Weighted & 1.037 & 0.848 & 0.730 & 2.286 & 3.208 \\ 
  1.286 & Unweighted & 1.126 & 0.924 & 0.771 & 2.746 & 2.957 \\ 
   \hline
1.357 & Weighted & 1.280 & 1.067 & 0.585 & 2.169 & 2.367 \\ 
  1.357 & Unweighted & 1.271 & 1.069 & 0.673 & 2.638 & 2.436 \\ 
   \hline
1.429 & Weighted & 1.485 & 1.286 & 0.459 & 2.075 & 1.768 \\ 
  1.429 & Unweighted & 1.398 & 1.206 & 0.575 & 2.526 & 2.008 \\ 
   \hline
1.500 & Weighted & 1.747 & 1.569 & 0.287 & 1.934 & 1.064 \\ 
  1.500 & Unweighted & 1.515 & 1.340 & 0.467 & 2.398 & 1.603 \\ 
   \hline
\end{tabular}
\endgroup
\caption{Simulation results for Model 4. Bolded values represent the better for a given $\lambda$ setting.} 
\label{tab:Model_4}
\end{table}
\begin{table}[H]
\centering
\begingroup\footnotesize
\begin{tabular}{cc|ccccc}
  \hline
$\lambda$ & Model & RMSE & MAE & Covg & Interval Width & Score \\ 
  \hline
1.000 & Weighted & 0.543 & 0.434 & 0.813 & 1.454 & 6.247 \\ 
  1.000 & Unweighted & 0.547 & 0.437 & 0.838 & 1.555 & 6.241 \\ 
   \hline
1.071 & Weighted & 0.544 & 0.435 & 0.808 & 1.422 & 6.255 \\ 
  1.071 & Unweighted & 0.545 & 0.436 & 0.826 & 1.493 & 6.263 \\ 
   \hline
1.143 & Weighted & 0.550 & 0.441 & 0.797 & 1.394 & 6.173 \\ 
  1.143 & Unweighted & 0.551 & 0.443 & 0.815 & 1.466 & 6.176 \\ 
   \hline
1.214 & Weighted & 0.558 & 0.446 & 0.786 & 1.383 & 6.068 \\ 
  1.214 & Unweighted & 0.568 & 0.457 & 0.798 & 1.461 & 5.954 \\ 
   \hline
1.286 & Weighted & 0.594 & 0.476 & 0.755 & 1.373 & 5.654 \\ 
  1.286 & Unweighted & 0.607 & 0.492 & 0.776 & 1.478 & 5.527 \\ 
   \hline
1.357 & Weighted & 0.639 & 0.514 & 0.712 & 1.376 & 5.117 \\ 
  1.357 & Unweighted & 0.661 & 0.538 & 0.731 & 1.495 & 4.940 \\ 
   \hline
1.429 & Weighted & 0.720 & 0.587 & 0.632 & 1.354 & 4.287 \\ 
  1.429 & Unweighted & 0.743 & 0.615 & 0.653 & 1.480 & 4.140 \\ 
   \hline
1.500 & Weighted & 0.803 & 0.666 & 0.558 & 1.353 & 3.578 \\ 
  1.500 & Unweighted & 0.820 & 0.688 & 0.596 & 1.499 & 3.561 \\ 
   \hline
\end{tabular}
\endgroup
\caption{Simulation results for Model 5. Bolded values represent the better for a given $\lambda$ setting.} 
\label{tab:Model_5}
\end{table}

\begin{thebibliography}{}

\bibitem[Barber et~al., 2019]{Barber2019}
Barber, R.~F., Candes, E.~J., Ramdas, A., and Tibshirani, R.~J. (2019).
\newblock Conformal prediction under covariate shift.
\newblock {\em arXiv preprint arXiv:1904.06019}.

\bibitem[Breiman, 2001]{Breiman2001}
Breiman, L. (2001).
\newblock Random forests.
\newblock {\em Machine learning}, 45(1):5--32.

\bibitem[Cangialosi et~al., 2018]{IrmaSummary}
Cangialosi, J.~P., Latto, A.~S., and Berg, R. (2018).
\newblock Hurricane irma.
\newblock In {\em National Hurricane Center Tropical Cyclone Report}.

\bibitem[Chen, 2014]{Chen2014}
Chen, S. (2014).
\newblock {\em Imputation of missing values using quantile regression}.
\newblock PhD thesis, Iowa State University.

\bibitem[Coleman et~al., 2019]{Coleman2019}
Coleman, T., Peng, W., and Mentch, L. (2019).
\newblock Scalable and efficient hypothesis testing with random forests.
\newblock {\em arXiv preprint arXiv:1904.07830}.

\bibitem[Friedman et~al., 2001]{Friedman2001b}
Friedman, J., Hastie, T., and Tibshirani, R. (2001).
\newblock {\em The elements of statistical learning}, volume~1.
\newblock Springer series in statistics New York, NY, USA:.

\bibitem[Friedman, 1991]{Friedman1991}
Friedman, J.~H. (1991).
\newblock Multivariate adaptive regression splines.
\newblock {\em The annals of statistics}, pages 1--67.

\bibitem[Guikema and Quiring, 2012]{Guikema2012}
Guikema, S.~D. and Quiring, S.~M. (2012).
\newblock Hybrid data mining-regression for infrastructure risk assessment
  based on zero-inflated data.
\newblock {\em Reliability Engineering \& System Safety}, 99:178--182.

\bibitem[He et~al., 2017]{He2017}
He, J., Wanik, D.~W., Hartman, B.~M., Anagnostou, E.~N., Astitha, M., and
  Frediani, M.~E. (2017).
\newblock Nonparametric tree-based predictive modeling of storm outages on an
  electric distribution network.
\newblock {\em Risk Analysis}, 37(3):441--458.

\bibitem[Kanamori et~al., 2009]{Kanamori2009}
Kanamori, T., Hido, S., and Sugiyama, M. (2009).
\newblock A least-squares approach to direct importance estimation.
\newblock {\em Journal of Machine Learning Research}, 10(Jul):1391--1445.

\bibitem[{Landsea} and {Franklin}, 2013]{NHC2013}
{Landsea}, C.~W. and {Franklin}, J.~L. (2013).
\newblock {Atlantic Hurricane Database Uncertainty and Presentation of a New
  Database Format}.
\newblock {\em Monthly Weather Review}, 141:3576--3592.

\bibitem[Liu et~al., 2005]{Liu2005}
Liu, H., Davidson, R.~A., Rosowsky, D.~V., and Stedinger, J.~R. (2005).
\newblock Negative binomial regression of electric power outages in hurricanes.
\newblock {\em Journal of infrastructure systems}, 11(4):258--267.

\bibitem[Liu and Meng, 2016]{Liu2016}
Liu, K. and Meng, X.-L. (2016).
\newblock There is individualized treatment. why not individualized inference?
\newblock {\em Annual Review of Statistics and Its Application}, 3:79--111.

\bibitem[Meinshausen, 2006]{Meinshausen2006}
Meinshausen, N. (2006).
\newblock Quantile regression forests.
\newblock {\em Journal of Machine Learning Research}, 7(Jun):983--999.

\bibitem[Mentch and Hooker, 2016]{Mentch2016}
Mentch, L. and Hooker, G. (2016).
\newblock Quantifying uncertainty in random forests via confidence intervals
  and hypothesis tests.
\newblock {\em The Journal of Machine Learning Research}, 17(1):841--881.

\bibitem[Mentch and Hooker, 2017]{Mentch2017}
Mentch, L. and Hooker, G. (2017).
\newblock Formal hypothesis tests for additive structure in random forests.
\newblock {\em Journal of Computational and Graphical Statistics},
  26(3):589--597.

\bibitem[Pasqualini et~al., 2017]{Pasqualini2017}
Pasqualini, D., Kaufeld, K., and Dorn, M.~F. (2017).
\newblock Electric power outage forecasting model.
\newblock Technical report, Los Alamos National Laboratory.

\bibitem[Peng et~al., 2019]{Peng2019}
Peng, W., Coleman, T., and Mentch, L. (2019).
\newblock Asymptotic distributions and rates of convergence for random forests
  and other resampled ensemble learners.
\newblock {\em arXiv preprint arXiv:1905.10651}.

\bibitem[Powers et~al., 2015]{Powers2015}
Powers, S., Hastie, T., Tibshirani, R., et~al. (2015).
\newblock Customized training with an application to mass spectrometric imaging
  of cancer tissue.
\newblock {\em The Annals of Applied Statistics}, 9(4):1709--1725.

\bibitem[Reddi et~al., 2015]{Reddi2015}
Reddi, S.~J., Poczos, B., and Smola, A. (2015).
\newblock Doubly robust covariate shift correction.
\newblock In {\em Twenty-Ninth AAAI Conference on Artificial Intelligence}.

\bibitem[Shimodaira, 2000]{Shimodaira2000}
Shimodaira, H. (2000).
\newblock Improving predictive inference under covariate shift by weighting the
  log-likelihood function.
\newblock {\em Journal of statistical planning and inference}, 90(2):227--244.

\bibitem[Stekhoven and B{\"u}hlmann, 2011]{Stekhoven2011}
Stekhoven, D.~J. and B{\"u}hlmann, P. (2011).
\newblock Missforest—non-parametric missing value imputation for mixed-type
  data.
\newblock {\em Bioinformatics}, 28(1):112--118.

\bibitem[Sugiyama et~al., 2007]{Sugiyama2007}
Sugiyama, M., Krauledat, M., and M{\~A}{\v{z}}ller, K.-R. (2007).
\newblock Covariate shift adaptation by importance weighted cross validation.
\newblock {\em Journal of Machine Learning Research}, 8(May):985--1005.

\bibitem[Sugiyama and M{\"u}ller, 2005]{Sugiyama2005}
Sugiyama, M. and M{\"u}ller, K.-R. (2005).
\newblock Input-dependent estimation of generalization error under covariate
  shift.
\newblock {\em Statistics \& Decisions}, 23(4/2005):249--279.

\bibitem[Therneau et~al., 1997]{Therneau1997}
Therneau, T.~M., Atkinson, E.~J., et~al. (1997).
\newblock An introduction to recursive partitioning using the rpart routines.

\bibitem[Tokdar and Kass, 2010]{Tokdar2010}
Tokdar, S.~T. and Kass, R.~E. (2010).
\newblock Importance sampling: a review.
\newblock {\em Wiley Interdisciplinary Reviews: Computational Statistics},
  2(1):54--60.

\bibitem[Wager and Athey, 2017]{Wager2017}
Wager, S. and Athey, S. (2017).
\newblock Estimation and inference of heterogeneous treatment effects using
  random forests.
\newblock {\em Journal of the American Statistical Association}.

\bibitem[Wager et~al., 2014]{Wager2014}
Wager, S., Hastie, T., and Efron, B. (2014).
\newblock Confidence intervals for random forests: The jackknife and the
  infinitesimal jackknife.
\newblock {\em The Journal of Machine Learning Research}, 15(1):1625--1651.

\bibitem[Wanik et~al., 2015]{Wanik2015}
Wanik, D., Anagnostou, E., Hartman, B., Frediani, M., and Astitha, M. (2015).
\newblock Storm outage modeling for an electric distribution network in
  northeastern usa.
\newblock {\em Natural Hazards}, 79(2):1359--1384.

\bibitem[Willoughby et~al., 2007]{Willoughby2007}
Willoughby, H.~E., Rappaport, E., and Marks, F. (2007).
\newblock Hurricane forecasting: The state of the art.
\newblock {\em Natural Hazards Review}, 8(3):45--49.

\bibitem[Wright and Ziegler, 2015]{Wright2015}
Wright, M.~N. and Ziegler, A. (2015).
\newblock Ranger: a fast implementation of random forests for high dimensional
  data in c++ and r.
\newblock {\em arXiv preprint arXiv:1508.04409}.

\bibitem[Xu et~al., 2016]{Xu2016}
Xu, R., Nettleton, D., and Nordman, D.~J. (2016).
\newblock Case-specific random forests.
\newblock {\em Journal of Computational and Graphical Statistics},
  25(1):49--65.

\end{thebibliography}
\end{document}